\documentclass[lettersize, 10 pt, journal]{IEEEtran}  

\IEEEoverridecommandlockouts                              

\usepackage{amsmath}
\usepackage{amssymb}
\usepackage{tikz}
\usepackage{ulem}
\renewcommand{\baselinestretch}{.99}

\usepackage{amsthm}
\usepackage[textsize=scriptsize]{todonotes}
\usepackage{lipsum}

\usepackage{xspace}
\usepackage{array}
\usepackage{multirow}
\usepackage{afterpage}
\usepackage{colortbl}
\definecolor{cello}{HTML}{ffe6cc}

\usepackage{graphics} 
\usepackage{etoolbox}

\newtheorem{myremark}{Remark}
\newtheorem{myassumption}{Assumption}
\newtheorem{mylem}{Lemma}
\newtheorem{mydef}{Definition}
\newtheorem{myprop}{Proposition}

\newtheorem{mytheorem}{Theorem}

\usepackage{color}
\usepackage{tabularray}
\definecolor{DoveGray}{rgb}{0.4,0.4,0.4}
\definecolor{Woodsmoke}{rgb}{0.094,0.101,0.105}
\newcommand{\refineCBF}{\textsc{refineCBF}\xspace}
\newcommand{\saferefineCBF}{\textsc{safeadapt-refineCBF}\xspace}
\newcommand{\algname}{\textsc{HJ-Patch}\xspace}
\usepackage{xcolor}
\usepackage{accents}
\usepackage{centernot}
\usepackage{algorithm}
\usepackage{algorithmic}
\usepackage{subcaption} 
\usepackage{mathtools}
\usepackage[font=small,labelfont=bf]{caption}
\usepackage{makecell}

\newcommand{\R}{\mathbb{R}}

\newcommand{\statedim}{n}
\newcommand{\ctrldim}{m}
\newcommand{\distdim}{o}
\newcommand{\state}{x}
\newcommand{\trajectory}{\mathbf{\state}}
\newcommand{\ctrl}{u}
\newcommand{\ctrlsignal}{\mathbf{\ctrl}}
\newcommand{\dist}{d}
\newcommand{\distsignal}{\mathbf{\dist}}
\newcommand{\diststrategy}{\mathfrak{d}}
\newcommand{\stateset}{\mathcal{X}}
\newcommand{\ctrlset}{\mathcal{U}}
\newcommand{\distset}{\mathcal{D}}
\newcommand{\distsignalset}{\mathbb{D}}
\newcommand{\ctrlsignalset}{\mathbb{U}}
\newcommand{\diststrategyset}{\Xi}

\newcommand{\cldyn}{F}
\newcommand{\oldyn}{f}
\newcommand{\ctrldyn}{g}
\newcommand{\distdyn}{w}

\newcommand{\constraintset}{\mathcal{L}}
\newcommand{\vf}{h}

\newcommand{\gi}[1]{\gamma_{#1}}
\newcommand{\decayrate}{\lambda}
\newcommand{\lie}{L_{\cldyn}}
\newcommand{\lieop}[1]{L_{{#1}}}
\newcommand{\lieopt}{L^*_{\cldyn}}
\newcommand{\safeset}{\mathcal{H}}
\newcommand{\safesetboundary}{\partial\safeset}
\newcommand{\nominalctrl}{\hat{\pi}(\state)}

\newcommand{\constraintfunc}{\ell}
\newcommand{\viability}[1]{{\mathcal{S}(#1)}}
\newcommand{\initialvf}{\vf^0}

\newcommand{\converged}[1]{{#1}^*}
\newcommand{\iter}[2]{{#1^{(#2)}}}
\newcommand{\boundarycells}{\partial^\zeta\safeset}
\newcommand{\activeset}{R}
\newcommand{\neighbors}[1]{X^{#1}}
\newcommand{\neighbororder}{p}

\newcommand{\oraclecells}{C}

\definecolor{mygray}{HTML}{959FC4}
\definecolor{onlinegray}{gray}{0.8}
\definecolor{onlinegray}{HTML}{519797}
\newcommand{\quasisynonym}{discrete-approximated}
\newcommand{\Quasisynonym}{Discrete-approximated}

\usepackage{pifont}
\newcommand{\cxmark}{\color{red}{\ding{55}}}
\newcommand{\greentick}{\color{green}{\ding{51}}}

\makeatletter
\let\NAT@parse\undefined
\makeatother
\usepackage{hyperref}

\title{\LARGE \bf
Refining Almost-Safe Value Functions on the Fly
}

\author{Sander Tonkens, Sosuke Kojima, Chenhao Liu, Judy Masri, and Sylvia Herbert
\thanks{The authors are with University of California, San Diego \{\href{mailto:sander@ucsd.edu}{sander}, \href{mailto:skojima@ucsd.edu}{skojima}, \href{mailto:chl238@ucsd.edu@ucsd.edu}{chl238}, \href{mailto:jmohamad@ucsd.edu}{jmohamad}, \href{mailto:sherbert@ucsd.edu}{sherbert}\}@ucsd.edu.}}

\begin{document}
\normalem

\maketitle

\begin{abstract}
Control Barrier Functions (CBFs) are a powerful tool for ensuring robotic safety, but designing or learning valid CBFs for complex systems is a significant challenge. While Hamilton-Jacobi Reachability provides a formal method for synthesizing safe value functions, it scales poorly and is typically performed offline, limiting its applicability in dynamic environments. 
This paper bridges the gap between offline synthesis and online adaptation.
We introduce \refineCBF for refining an approximate CBF—whether analytically derived, learned, or even unsafe—via warm-started HJ reachability. We then present its computationally efficient successor, \algname, which accelerates this process through localized updates. Both methods guarantee the recovery of a safe value function and can ensure monotonic safety improvements during adaptation.
Our experiments validate our framework's primary contribution: in-the-loop, real-time adaptation, in simulation (with detailed value function analysis) and on physical hardware. 
Our experiments on ground vehicles and quadcopters show that our framework can successfully adapt to sudden environmental changes, such as new obstacles and unmodeled wind disturbances, providing a practical path toward deploying formally guaranteed safety in real-world settings.

\end{abstract}
\thispagestyle{empty}
\pagestyle{empty}
\section{Introduction}
The widespread adoption of learning-based modules for perception and control has unlocked new capabilities in robotics, from dynamic locomotion~\cite{MikiLeeEtAl2022} to complex navigation~\cite{ShahSridharEtAl2023}. 
However, their effectiveness is often tightly coupled to their training conditions, limiting their ability to adapt to uncertain or evolving scenarios.
This lack of adaptability presents a significant barrier to deployment in safety-critical applications, where environmental changes or distribution shifts can lead to catastrophic failures. 
This necessitates safety modules capable of adapting to changing conditions online in real-time, directly at the control level.

A promising approach for realizing such a module is the safety filter: a component that monitors the actions from a decision-making module and intervenes only when necessary to prevent failure~\cite{BrunkeGreeffEtAl2021, Wabersich2023DataDrivenSF}. 
Many of these filters are built upon a safety value function, a scalar field over the state space designed to encode safety information. 
Control Barrier Functions (CBF)~\cite{AmesGrizzleEtAl2014} are a prominent example. 
Typically, the function's value at a given state indicates the system's safety margin, while its gradient delineates the set of control actions that maintain or improve safety.
The filter leverages the function’s value and gradient to constrain the nominal control input to the set of safe control actions at the current state. 
However, the central challenge is designing a valid safe value function, i.e., one that guarantees persistent trajectory safety from purely pointwise enforcement.

This challenge has been approached from two main directions, each with significant drawbacks:
\begin{enumerate}
    \item \textbf{Constructive methods}, such as Hamilton-Jacobi (HJ) Reachability~\cite{BansalChenEtAl2017b}, offer formal guarantees by numerically solving for the value function. However, they scale poorly with state dimensionality.
    \item \textbf{Approximate methods}, which leverage function approximators like neural networks, offer scalability but sacrifice formal guarantees and do not generalize easily.
\end{enumerate}
Crucially, both paradigms struggle to adapt to real-world variations encountered post-deployment, such as sim-to-real gaps or unexpected changes in system dynamics.

In this work, we argue that the constructive and approximate methods for generating safe value functions are not mutually exclusive but are, in fact, complementary. 
We propose to bridge the gap between formal guarantees and scalability by using a data-driven approximation as a warm start for a formal, constructive refinement process that leverages HJ Reachability. 

\subsection{Related work}
Safety is a cornerstone of robotics research, with methods ranging from optimal control and model predictive control (MPC)~\cite{RosoliaBorrelli2018} to formal verification~\cite{AbateAhmedEtAl2021}. 
While optimal control can formally incorporate safety constraints, solving these problems for complex, nonlinear systems is often computationally intractable, sacrificing guarantees for real-time performance. 
In contrast, optimal control-inspired approaches, such as constrained Reinforcement Learning, penalize constraint violations or provide guarantees solely in expectation, limiting their use in safety-critical settings. 
This has led to the rise of safety filters, standalone modules that minimally modify a primary policy's control output to ensure safety~\cite{Wabersich2023DataDrivenSF}.

A popular approach to designing these filters is to use a safety value function—a scalar field over the state space that encodes safety information. Current research on these value functions largely falls into two complementary categories.

\paragraph{Control Barrier Functions (CBFs)}
CBFs provide a principled way to enforce safety online. By constraining the rate at which a system can approach a failure boundary, a CBF can be used to formulate a simple, pointwise optimization problem that minimally modifies a nominal control input to guarantee safety~\cite{Ames2017ControlBF}. This enforcement mechanism is highly efficient to solve for control-affine systems, for which it reduces to a quadratic program.

However, the primary challenge for CBFs is characterizing safety—that is, finding a valid CBF in the first place. 
While they can be derived analytically for some systems, this is difficult for complex nonlinear dynamics, especially under external disturbances. 
To address this, learning-based approaches have been proposed to approximate CBFs from data~\cite{DawsonQinEtAl2021, RobeyHuEtAl2020}. 
These methods, while scalable, often lack formal guarantees and their reliability depends heavily on the quality and density of the training data. 
Approaches like backup CBFs~\cite{GurrietMoteEtAl2020} provide stronger guarantees but introduce significant computational complexity. Additionally, their overall performance hinges on the expert design of a fixed ``backup'' policy.

\paragraph{Hamilton-Jacobi Reachability (HJR)}
In contrast to CBFs, HJ Reachability provides a formal method to characterize safety. 
It formulates a worst-case optimal control problem to compute the set of all states from which failure is inevitable, yielding a value function with strong safety guarantees for general nonlinear systems with disturbances~\cite{Mitchell2005ATH}.

However, this approach faces two significant hurdles.
First, its traditional solution involves numerically solving a Hamilton-Jacobi partial differential equation (PDE) on a state-space grid. This method suffers from the curse of dimensionality, as its computational complexity grows exponentially with the number of state variables, typically limiting them to systems with fewer than six state dimensions~\cite{BansalChenEtAl2017b}. 
Numerous works have sought to improve the scalability of the original formulation directly by leveraging problem structure (e.g., via warm-starting~\cite{HerbertBansalEtAl2019}, decomposition~\cite{ChenHerbertEtAl}) or using learning-based approximations to eliminate state-space gridding (e.g., DeepReach~\cite{BansalTomlin2021}, and reinforcement learning~\cite{FisacLugovoyEtAl2019}). However, they often do so at the cost of approximation errors, extensive pre-computation, or overly conservative behavior.  
Second, and more critical for online use, HJR is typically used as a least-restrictive safety filter~\cite{BajcsyBansalEtAl2019}, where the optimal control is enforced only when near the safety boundary, often leading to either conservative or chattering control behavior. 

In contrast, another family of methods sidesteps the PDE entirely by trading generality for computational tractability. 
These HJ-inspired approaches solve a related but more constrained problem, such as using Sum-of-Squares (SoS) optimization for systems with polynomial dynamics~\cite{TedrakeManchesterEtAl2010, MajumdarTedrake2017} or using simpler geometric representations like zonotopes~\cite{AlthoffKrogh2011} which are efficient to propagate but may be overly conservative for general nonlinear systems. 
In contrast to the direct HJ approaches, these methods are often not deployed online with a safety filter, but e.g., generate safe ``funnels'' around trajectories~\cite{MajumdarAhmadiEtAl2013} or leverage pre-computed sets for collision checking~\cite{KousikVaskovEtAl2017, KousikVaskovEtAl2020}. 

\paragraph{Bridging the Gap}
Summarizing, CBF-based methods offer efficient enforcement but rely on a safe value function, while HJR-based methods provide formal characterization but typically lack efficient enforcement. 
For example,~\cite{SoSerlinEtAl2024} leverages HJ-based supervised rollouts for learning a CBF~\cite{SoSerlinEtAl2024}.
Choi et al.~\cite{ChoiLeeEtAl2021} first established a direct link between CBFs and HJR using 
Control Barrier Value Functions (CBVFs). 
They bridge this gap by using HJR-like computations (with discounting) to construct a valid CBVF; unfortunately, this approach inherits the high computational cost of HJR and does not leverage approximate solutions. 

Alternatively, \refineCBF~\cite{TonkensHerbert2022} and its extension \algname~\cite{TonkensToofanianEtAl2024}, leveraged approximate initial value functions (from e.g., a learned CBF) to warmstart a HJR-based formal refinement process. 
Summarized, these methods leverage an initial approximation (and an adaptive updating scheme) for faster convergence. 
However, they construct the value function offline for a given environment and are limited to numerical simulations. 
This work instead focuses on adapting value functions online to adapt to real-world changes to the environment or the system, while retaining the safety and convergence guarantees of~\cite{TonkensHerbert2022} and~\cite{TonkensToofanianEtAl2024}. 

\subsection{Contributions}

This work presents a formally-grounded adaptive safety framework that enables online, in-the-loop adaptation to sensed changes in the system and environment. 
Our algorithmic extensions to \refineCBF and \algname rely on iteratively refining any given value function until convergence while adapting to changes in the system and the environment.  
Specifically, our contributions are: 
\begin{itemize}
    \item A theoretical extension of the control-only \refineCBF~\cite{TonkensHerbert2022} method to formally guarantee safety for systems subject to unknown but bounded external disturbances. 
    \item Unified algorithmic framework spanning offline and in-the-loop refinement, including theoretical analysis and implications for online deployment. 
    \item A rigorous comparative analysis against typical CBF approaches (disjoint CBFs, backup CBFs), highlighting the benefits of principally refining online in realistic simulated experiments with real-time requirements.
    \item Extensive, novel experiments on hardware (mobile robots and quadcopter) that demonstrate real-time adaptation to online detection of unforeseen obstacles and wind disturbances.
\end{itemize}

\subsection{Organization}
The remainder of this paper is organized as follows. Section~\ref{sec:prelims} reviews the necessary background on value function-based safety. We then introduce our core algorithms and their theoretical guarantees: \refineCBF in Section~\ref{sec:refinecbf} and its computationally efficient counterpart, \algname, in Section IV~\ref{sec:hjpatch}. We validate our framework through extensive experiments, presenting detailed simulations in Section~\ref{sec:sim}, and hardware demonstrations on multiple robotic platforms in Section~\ref{sec:hardware}. Summarizing, Section~\ref{sec:future} provides use-cases and future work, followed by a Conclusion in Section~\ref{sec:conclusion}.

\section{Preliminaries}\label{sec:prelims}
We consider a control- and disturbance affine dynamical system
\begin{equation}\label{eq:dynamics}
    \dot \state = \cldyn(\state,\ctrl,\dist)\triangleq \oldyn(\state) + \ctrldyn(\state)\ctrl + \distdyn(\state) \dist,
\end{equation}
which describes a wide range of robots, 
with state $\state\in\stateset \subset \R^\statedim$, input $\ctrl \in \ctrlset \subset \R^\ctrldim$, and disturbance $\dist \in \distset \subset \R^\distdim$, where $\ctrlset$ and $\distset$ are polytopes. 
We assume the dynamics $\cldyn:\R^\statedim \times \ctrlset \times \distset \mapsto \R^\statedim$ are uniformly continuous, bounded, and Lipschitz continuous in $\state$. 
The Lipschitz continuity of the dynamics ensures the existence and uniqueness of system trajectories $\trajectory_{\state,t}^{\ctrlsignal, \distsignal}(s)$ starting from state $\state$ at time $t$ under a control signal $\ctrlsignal$ and disturbance signal $\distsignal$.
We denote the set of measurable functions $\ctrlsignal:\R_{\geq t} \to \ctrlset$ and $\distsignal:\R_{\geq t} \to \distset$ as $\ctrlsignalset$ and $\distsignalset$, representing the allowed control and disturbance signals.

Let $\constraintset\subseteq \stateset$ represent the constraint set, i.e. the set of non-failure states of the system~\eqref{eq:dynamics}. 
An element $\state\in\stateset$ is considered instantaneously \textit{safe} if $\state \in \constraintset$. 
Its complement, $\constraintset^C$, denotes the set of failure states. 
The goal of a safety filter is to keep the system within the constraint set $\constraintset$. 
In this work, the disturbance $\distsignal$ is determined in reaction to the control signal in the form of a disturbance strategy $\diststrategy:\ctrlsignalset \mapsto \distsignalset$. We restrict $\diststrategy$ to draw from a \emph{nonanticipative strategy}, which prohibits using future knowledge of the control signal to preserve causality~\cite{EvansSougandis1984}, with $\diststrategyset$ the set of all such strategies. 

A foundational concept for guaranteeing safety is control invariance.

\begin{mydef}[Robust Control Invariant Set]\label{def:CIset}
A set $\mathcal{C} \subseteq \stateset$ is a robust control invariant set if for any state $\state \in \mathcal{C}$  and initial time $t\in \R$, there exists a control signal $\ctrlsignal\in \ctrlsignalset$ such that for all disturbance strategies $\diststrategy \in \diststrategyset$, the system trajectory remains in the set, i.e., $\trajectory_{\state,t}^{\ctrlsignal, \diststrategy}(s) \in \mathcal{C}$ for all $s \ge t$.
\end{mydef}
By this definition, the empty set $\emptyset$ is also considered control invariant, as the required condition holds vacuously. 
For any given constraint set $\constraintset$, we are often interested in finding the largest possible safe region within it. This is known as the viability kernel.

\begin{mydef}[Viability Kernel,~\cite{AubinBayenEtAl2011}]\label{def:viability} 
For any set $\mathcal{K}$, let the viability kernel $\viability{\mathcal{K}}$ be such that $\viability{\mathcal{K}} \subseteq \mathcal{K}$ and if there exists $\mathcal{C} \subseteq \mathcal{K}$ such that $\mathcal{C}$ is control invariant, then $\mathcal{C} \subseteq \viability{\mathcal{K}}$, hence $\viability{\mathcal{K}}$ is the largest control invariant set in $\mathcal{K}$.
\end{mydef}

\subsection{Control Barrier Functions}

Let a value function $\vf: \stateset \mapsto \R$ be Lipschitz continuous and let $\safeset:=\{\state \mid h(\state) \geq 0\}\subseteq\stateset$ be the 0-superlevel set of $\vf$. 
We define the Lie derivative $\lie\vf(\state, \ctrl, \dist):=\langle \frac{\partial h}{\partial \state}, \cldyn(\state, \ctrl, \dist)\rangle$ and the Hamiltonian $\lieopt \vf(\state):=\inf_{\dist\in\distset} \sup_{\ctrl\in\ctrlset}\lie \vf(\state)$. For non-continuously differentiable functions, the Lie derivative can be formulated through the Clarke generalized gradient~\cite{Clarke1998}.

Nagumo's theorem~\cite{Blanchini1999} provides a condition for control invariance: assume $\frac{\partial \vf}{\partial \state}\neq 0$ for all $\state \in \safesetboundary:=\{\state \mid \vf(\state) = 0\}$, then $\safeset$ is control invariant if and only if
\begin{equation}\label{eq:nagumo}
    \lieopt \vf(\state) \geq 0 \text{ for all } \state \in \safesetboundary.
\end{equation}

In this paper, a safe value function is considered to be any function $\vf$ that satisfies Nagumo's theorem, and whose 0-superlevel set $\safeset$ is a subset of $\constraintset$, $\safeset\subseteq \constraintset$. 
The CBF condition extends Nagumo's theorem from the boundary to the interior of the set $\safeset$ with a Lyapunov-like condition, ensuring the state does not approach the boundary too quickly.

\begin{mydef}[Control Barrier Function~\cite{Ames2017ControlBF}]\label{def:cbf}
Let $\safeset$ denote the 0-superlevel set of a continuously differentiable value function $\vf: \stateset \mapsto \R$, then $\vf(\state)$ is a control barrier function for \eqref{eq:dynamics} if there exists an extended class $\mathcal{K}$ function $\alpha$ and a set $\mathcal{C}$ with $\safeset \subseteq \mathcal{C} \subset \R^\statedim$ such that
\begin{equation}\label{eq:cbf}
    \lieopt \vf(\state) \geq -\alpha(\vf(\state))
\end{equation}
for all $\state \in \mathcal{C}$. 
\end{mydef}
A CBF defines a control invariant set $\safeset$~\cite{AmesCooganEtAl2019}. 
For completeness, we define the state-dependent allowable control range for a CBF as follows:
\begin{equation}\label{eq:cbf_online}
    \mathcal{G}_\vf(\state) := \left\{\ctrl \in \ctrlset \bigm\vert \min_{\dist\in\distset}\lie{\vf}(\state,\ctrl,\dist) + \alpha(\vf(\state)) \geq 0\right\}
\end{equation}
Any choice of control law for which $\ctrl \in \mathcal{G}$ will ensure that the system~\eqref{eq:dynamics} remains in $\safeset$ indefinitely.

The popularity of CBFs can be attributed to the ease with which they can be used in a safety filter online through an optimization problem.
At each time step, a nominal (safety-agnostic) policy $\nominalctrl$ is passed through a filter that solves:
\begin{equation}\label{eq:online-cbf-qp-full}
    \begin{split}
        \ctrl^*(\state)= \> \arg \min \limits_{\ctrl} \quad &\lVert \ctrl - \nominalctrl \rVert_2^2 \\
        \text{subject to}\quad & \lieop{\ctrldyn} \vf(\state) \ctrl \geq - \gamma\vf(\state) -\lieop{f} \vf(\state)\\ 
        & \hspace{1.7cm}-\min_{\dist\in \distset}\lieop{\distdyn} \vf(\state) \dist \\
        & \ctrl \in \ctrlset.
    \end{split}
\end{equation}
In this paper we assume $\alpha(\state)=\gi{} \state$, for $\gi{} >0$. 
As the disturbance and control are independent for~\eqref{eq:dynamics}, we can first solve for the disturbance which decreases the value function most through an exhaustive search of its vertices, with the resulting optimization problem to solve for $\ctrl^*$ being a quadratic program (QP), which can be solved in real time ($>100$ Hz) for typical robotics applications.

The main challenge of this approach, however, lies in finding a valid function $\vf$ for which the QP~\eqref{eq:online-cbf-qp-full} is feasible for all states $\state$ in its 0-superlevel set.
To precisely discuss the quality of candidate CBFs often used in practice, we classify CBFs into 4 possible categories:
\begin{table}[H]
\caption{The 4 possible categories of CBFs and their properties.}
\label{tab:cbfs}
\centering
\setlength{\tabcolsep}{2pt}
\setlength{\extrarowheight}{1.1pt}
\begin{tabular}{c|c|c|c}
    & \thead{Characterizes safety: \\$\safeset \cap \constraintset = \emptyset$} & \thead{Characterizes control \\ invariance:~\eqref{eq:online-cbf-qp-full} $\forall \state$} &  \thead{Non-intrusive \\ safety filter} \\
    Invalid CBF & \greentick & \cxmark & \greentick \color{black}{$\slash$}\cxmark \\\vspace{-.1cm}
    Unsafe CBF & \cxmark & \greentick & \greentick \\\vspace{-.2cm}
    \thead{Conservative \\ (safe) CBF} & \greentick & \greentick & \cxmark \\
    \thead{Desired \\ (safe) CBF} & \greentick & \greentick & \greentick
\end{tabular}
\end{table}

This work proposes methods to refine a candidate CBF into a desired, provably safe one. 
The resulting functions, which are constructed using tools from HJ reachability theory, are technically Control Barrier Value Functions (CBVFs), which we detail in the following section. 

\subsection{HJR-based Control Barrier Value Functions}
Given the aforementioned constraint set $\constraintset$, we formulate an associated constraint function, $\constraintfunc(\state)$, such that $\constraintfunc(\state) \geq 0$ if and only if $\state \in \constraintset$, e.g., a (weighted) signed-distance function.
For a desired decay rate $\decayrate$, see e.g.~\eqref{eq:online-cbf-qp-full}, we define the associated CBVF~\cite{ChoiLeeEtAl2021}:
\begin{equation}\label{eq:cbvf_trajectory}
    \vf_{\decayrate}(\state,t) = \min_{\diststrategy \in \diststrategyset} \max_{\ctrlsignal \in \ctrlsignalset}\min_{s \in [t,0]}e^{\decayrate(s-t)} \constraintfunc\left(\trajectory_{\state,t}^{\ctrlsignal, \diststrategy}(s)\right).
\end{equation}
We note that $\vf_{\decayrate}$ is uniquely defined for each fixed value of $\decayrate \geq 0$, and for $\decayrate=0$ its definition matches the definition of the standard HJ reachability value function formulation~\cite{FisacChenEtAl2015}. 

It is important to note that any CBF is a CBVF, while the inverse does not hold. 
\begin{myremark}[Differentiating a CBVF from a CBF]
In addition to satisfying Eq.~\eqref{eq:cbf} for all states, a CBF $\vf(\state)$ is, by definition (Def.~\ref{def:cbf}), also continuously differentiable, i.e. $\vf(\state)\in C^1$. 
In contrast, a CBVF is differentiable almost everywhere, see~\cite{ChoiLeeEtAl2021}. 
\end{myremark}

With a value function defined, the system can maintain safety during deployment by incorporating the value function into a safety filter. 
Similarly to CBFs, we can define the admissible control set that retains (finite time) safety, where we distinguish between the offline decay rate $\decayrate$ and online decay rate $\gi{}$.
\begin{equation}\label{eq:cbvf_online_timevarying}
    \mathcal{G}_{\vf_\decayrate}^{\gi{}}(\state, s) = \left\{\ctrl \in \ctrlset \bigm\vert \min_{\dist\in\distset} \dot{\vf}_{\decayrate}(\state,s) +  \gi{}\vf_{\decayrate}(\state,s) \geq 0 \right\}
\end{equation}
Importantly, the discount factor used to compute the CBVF offline $\decayrate$ does not require matching the CBVF decay rate online $\gi{}$; 
Specifically:
\begin{myprop}
[Forward completeness and safety with larger online discount rate]\label{prop:larger_online}
Applying a discount factor $\gi{}$ online~\eqref{eq:online-cbf-qp-full} for a CBVF that is constructed~\eqref{eq:cbvf_trajectory} with $\decayrate{} \leq \gi{}$ maintains control invariance of the safe set.
\begin{proof}
By inspection of Eq~\eqref{eq:cbvf_online_timevarying}, if $\gi{} \geq \decayrate$, we have that $\mathcal{G}_{\vf_{\decayrate}}^{\gi{}}(\state,s) \supseteq \mathcal{G}_{\vf_{\decayrate}}^{\decayrate}(\state,s) \neq \emptyset$ for all $\state \in \safeset$ and $s \in [t, 0]$. 
Hence, applying a larger discount rate $\gi{}$ online maintains pointwise feasibility. 
In addition, at the boundary of the safe set, $\state \in \safesetboundary$, we have that $\vf_{\decayrate} = 0$ for all $\decayrate\geq 0$ (see~\cite{ChoiLeeEtAl2021}), hence  $\mathcal{G}_{\vf_{\decayrate}}^{\gi{}}(\state,s) = \mathcal{G}_{\vf_{\decayrate}}^{\decayrate}(\state,s)$ by inspection of~\eqref{eq:cbvf_online_timevarying}. 
Combined, this ensures the same safety guarantees are maintained when applying a larger discount rate $\gi{}$ online.
\end{proof}
\end{myprop}

This enables setting $\decayrate=0$ during construction and applying any (positive) decay rate $\gi{}$ online. 
This is particularly useful as applying a non-zero discount factor offline is limited to finite-time safety problems~\cite{ChoiLeeEtAl2021}, whereas a discount factor $\decayrate=0$ holds for infinite-time safety, which this work focuses on.
Additionally, considering $\decayrate=0$ allows us to freely choose any $\gi{}>0$ for safety enforcement online while providing infinite-time safety guarantees. 
The associated admissible control set for a converged $\vf$, i.e. $D_t \vf=0$, is as follows:
\begin{equation}\label{eq:cbvf_online}
    \mathcal{G}^{\gi{}}(\state) = \left\{\ctrl \in \ctrlset \bigm\vert \min_{\dist\in\distset} \lie{\vf}(\state) +  \gi{}\vf(\state) \geq 0 \right\}.
\end{equation}

\subsection{Solutions to HJR value functions}
To solve for~\eqref{eq:cbvf_trajectory}, we present the associated continuous solution and its derived time discretized dynamic programming-based solution. 
For readability, we will drop the exponential term $\decayrate$, i.e. $\decayrate=0$, and refer to the equivalent formulas in~\cite{ChoiLeeEtAl2021} for the formulation including $\decayrate$. 
The value function for the $\gi{}=0$ case is defined as in~\cite{FisacChenEtAl2015}: 
\begin{equation}\label{eq:vf_trajectory}
    \vf(\state,t) = \min_{\diststrategy \in \diststrategyset} \max_{\ctrlsignal \in \ctrlsignalset}\min_{s \in [t,0]} \constraintfunc\left(\trajectory_{\state,t}^{\ctrlsignal, \diststrategy}(s)\right).
\end{equation}

We distinguish two popular formulations of HJ reachability to solve the safety scenario described in this paper. 
The first formulation, a single-boundary partial differential equation (PDE), forces contraction of the value function, and recovers the largest control invariant subset of the initial value function's 0-level set.
The second, a variational inequality (VI), recovers a control invariant set which is a subset of $\constraintset$.
In the classic HJR setting in which the terminal boundary condition is the constraint function, $\vf_b(\state)=\constraintfunc(\state)$, these two formulations are equivalent.

\subsubsection{Single-Boundary PDE Formulation \cite{MitchellBayenEtAl2005}}
The value function $\vf(\state,t)$ is the viscosity solution to the following Hamilton-Jacobi-Isaacs PDE (HJI-PDE):
\begin{align}\label{eq:pde_continuous}
    \begin{split}
        D_t \vf(\state,t) + \min\bigl\{0, \max_{\ctrl\in \ctrlset}\min_{\dist\in\distset} &D_x\vf(\state,\ctrl, \dist, \,t) \cdot \cldyn(\state,\ctrl,\dist) \bigr\}=0 \nonumber
    \end{split} \\
    \vf(x,0) &= \vf_b(x).
\end{align}

The discrete time equivalent of~\eqref{eq:pde_continuous} uses the following dynamic programming update rule:
\begin{equation}\label{eq:pde_dp}
    \begin{split}
        \iter{\vf}{k+1}(\state) = \iter{\vf}{k}(\state) + \iter{\Delta}{k} \min\{0, \lieopt \iter{\vf}{k}(\state)\},
    \end{split}
\end{equation}
with $\iter{\Delta}{k}$ denoting the time-step which is dynamically updated based on the magnitude of the Hamiltonians $\lieopt \iter{\vf}{k}(\state)$, see~\cite{Mitchell2008} and~\eqref{eq:finite_difference} for details. 
Eqs.~\eqref{eq:pde_continuous} and~\eqref{eq:pde_dp} define contractive mappings over time through the min-with-zero operation.

\subsubsection{Variational Inequality Formulation~\cite{FisacChenEtAl2015}}
The value function $\vf(\state,t)$ is the viscosity solution to the following HJI-VI:
\begin{align}\label{eq:vi_continuous}
    \begin{split}
        0 = \min\bigl\{&\constraintfunc(\state) - \vf(\state,t), \\
        & D_t \vf(\state,t) + \max_{\ctrl\in \ctrlset}\min_{\dist\in\distset} D_x\vf(\state,\ctrl, \dist, \,t) \cdot \cldyn(\state,\ctrl,\dist)\bigr\}\nonumber
    \end{split}
    \\
    &\vf(x,0) = \vf_b(x).
\end{align}

The discrete time equivalent of~\eqref{eq:vi_continuous} uses the following dynamic programming update rule:
\begin{equation}\label{eq:vi_dp}
    \begin{split}
        \iter{\vf}{k+1}(\state) = \min\left\{\constraintfunc(\state), \iter{\vf}{k}(\state) + \iter{\Delta}{k}\lieopt \iter{\vf}{k}(\state)\right\}. 
    \end{split}
\end{equation}

In practice, we solve~\eqref{eq:pde_dp} and~\eqref{eq:vi_dp} through spatial discretization on a pre-defined set of grid points as no closed-form solution exists. 
The Hamiltonian $\lieopt \iter{\vf}{k}(\state)$ is typically approximated using finite difference methods, namely (weighted) essentially non-oscillatory ((W)ENO) schemes~\cite{Shu1998} as follows:
\begin{equation}\label{eq:finite_difference}
    \hspace{-.2cm}\lie \vf(\state , \ctrl, \dist) := 
    \langle \nabla \vf(\state), \cldyn(\state,\ctrl,\dist)\rangle 
    \approx
    \langle \delta^p_{\Delta}(\vf), \cldyn(\state,\ctrl, \dist)\rangle,
\end{equation}

with $\delta^\neighbororder_{\Delta}$ a finite difference operator of order $\neighbororder$, i.e. the operator includes $\neighbororder$ neighbors in every dimension. 
In addition, the Hamiltonian $\lieopt \vf(\state)$ is computed through an exhaustive search of the vertices of $\ctrlset$ and $\distset$.
\begin{myremark}[Convergence of CBVF]
    A stationary solution to~\eqref{eq:pde_continuous} and~\eqref{eq:vi_continuous}, or equivalently, if~\eqref{eq:pde_dp} and~\eqref{eq:vi_dp} converge ($k \to \infty$), the result characterizes an infinite-time CBVF, which encodes safety for an infinite duration. 
    We drop the time-dependence for such a value function $\vf(\state)$.
\end{myremark}

\section{Refining CBFs}\label{sec:refinecbf}
This section details our primary contribution: a framework for refining a candidate value function $\initialvf$, using the HJI-VI formulation~\eqref{eq:vi_continuous}. 
This approach leverages an initial guess—such as a learned neural network or a hand-designed function—to ``warm-start" the computation, enabling faster convergence and online adaptation. 
We first establish the theoretical guarantees for our method in a static environment. 
We then present two algorithms: \refineCBF, a direct implementation of the refinement process, and \saferefineCBF, a variant with stronger intermediate safety guarantees, at the cost of added conservativeness. 
Finally, we discuss the practical implications of using these algorithms for in-the-loop adaptation in dynamic environments.

\subsection{Guarantees underlying refining CBFs}\label{subsec:theory_refinecbf}
We begin by establishing the theoretical guarantees of our approach in a static environment. 
In this section, we first prove our main result: that the refinement process, when warm-started with a candidate function $\initialvf$ is guaranteed to converge to a valid and safe CBVF (Theorem~\ref{thm:valid_CBVF}).
We then analyze the properties of the value function during convergence, establishing conditions under which safety is preserved throughout the refinement process (Lemma~\ref{lem:retain_safety}).

We define the boundary function of the HJI-VI~\eqref{eq:vi_continuous} as the candidate value function, $\vf_b=\initialvf$, shifting the objective from a standard safety evaluation in~\eqref{eq:vf_trajectory} to include a recursive refinement of the candidate value function. This defines the HJI-VI whose unique viscosity solution is the following trajectory-based value function:
\begin{equation}\label{eq:refinecbf_trajectory}
        \begin{split}
            \vf(\state,t) = \min_{\diststrategy \in \diststrategyset} \max_{\ctrlsignal \in \ctrlsignalset} \min\Bigg\{
            &
            \min_{s \in [t,0)}\constraintfunc\left(\trajectory_{\state,t}^{\ctrlsignal, \diststrategy}(s)\right), 
            \\&
            \initialvf\left(\trajectory_{\state,t}^{\ctrlsignal, \diststrategy}(0)\right)\Bigg\}.
        \end{split}
\end{equation}

To facilitate the proofs, we assume the following:
\begin{myassumption}[The candidate CBF is pointwise conservative with respect to the constraint function]\label{ass:conservative_guess} For all $\state$, we assume the candidate CBF $\initialvf(\state)$ satisfies $\initialvf(\state)\leq \constraintfunc(\state)$.
\end{myassumption}

As outlined in~\cite{TonkensHerbert2022}, Assumption~\ref{ass:conservative_guess} is not strictly required, but facilitates the proofs below. 
It can be trivially satisfied by defining a modified candidate CBF $\initialvf(\state)=\min\{\initialvf(\state), \constraintfunc(\state)\}$.

Theorem~\ref{thm:valid_CBVF} below is the key theoretical underpinning of \refineCBF.

\begin{mytheorem}[Convergence to a safe CBVF]\label{thm:valid_CBVF}
If~\eqref{eq:refinecbf_trajectory} converges to $\converged{\vf}(\state)$ as $t\to-\infty$, then $\converged{\vf}(\state)$ is a safe CBVF for all $\state\in\converged{\safeset}$, its 0-superlevel set. 
\end{mytheorem}

The proof of Theorem~\ref{thm:valid_CBVF} relies on the following three foundational lemmas, which establish that the converged value function's 0-superlevel set is control invariant (Lemma~\ref{lem:CI_refineCBF}), a subset of the viability kernel (Lemma~\ref{lem:viability_subset_refineCBF}), and that the function itself satisfies the CBF inequality (Lemma~\ref{lem:cbvf_refinecbf}).

\begin{mylem}[Convergence to a control invariant set]\label{lem:CI_refineCBF}
    If~\eqref{eq:refinecbf_trajectory} converges to $\converged{\vf}(\state)$ as $t \to -\infty$, i.e. $\converged{\vf}(\state)$ characterizes a stationary solution of~\eqref{eq:vi_continuous}, then $\converged{\safeset}$ is control invariant, i.e. its Hamiltonian  $\lieopt{\converged{\vf}}(\state) \geq 0$ for all $\state \in \converged{\safesetboundary}$.

    \begin{proof}
        The stationary version of the HJI-VI, is as follows, see~\cite{FialhoGeorigou1999}:
        \begin{equation}\label{eq:stationary_vi}
            \hspace{-.1cm}\min\left\{\constraintfunc(\state) - \vf(\state), \max_{\ctrl\in\ctrlset}\min_{\dist\in\distset} D_{\state}\vf(\state, \ctrl, \dist) \cdot \cldyn(\state,\ctrl,\dist)\right\}=0.
        \end{equation}
        Upon inspection of the second term of the minimum operator, any $\converged{\vf}(\state)$ satisfying~\eqref{eq:stationary_vi} satisfies $\lieopt{\converged{\vf}}(\state)\geq 0$ for all $\state$. 
        Particularly, this holds for all $\state \in \converged{\safesetboundary}$, thus  $\converged{\safeset}$ is a control invariant set by Nagumo's theorem,~\eqref{eq:nagumo}, see~\cite{Blanchini1999}.
    \end{proof}
\end{mylem}

\begin{mylem}[Convergence to a subset of the viability kernel]\label{lem:viability_subset_refineCBF}
If~\eqref{eq:refinecbf_trajectory} converges to $\converged{\vf}(\state)$ as $t \to -\infty$, its associated 0-superlevel set $\converged{\safeset}$ is a subset of the viability kernel of the constraint set $\viability{\constraintset}$, i.e. $\converged{\safeset}\subseteq\viability{\constraintset}$.    
    \begin{proof}
        This proof builds on the proof of Theorem 1 in~\cite{HerbertBansalEtAl2019}. 
        We (a) leverage that standard HJ Reachability recovers a value function whose 0-superlevel set corresponds to the viability kernel~\cite{Aubin1991}, and then (b) prove that for any time $t\in (-\infty, 0]$ the warmstarted value function is pointwise upper bounded by standard HJ reachability's value function. 
        For the purposes of this proof we define the CBVF obtained when refining a CBF $\initialvf$, i.e. through~\eqref{eq:refinecbf_trajectory}, as $\vf(\state, t; \initialvf)$.
        In contrast, we define the CBVF obtained using standard HJ reachability, i.e. through~\eqref{eq:vf_trajectory}, as $\vf(\state,t; \constraintfunc)$. 
        Then, we aim to show that $\vf(\state,t;\initialvf) \leq \vf(\state,t;\constraintfunc)$ for all $\state, t$. 
        We observe:
        \begin{align*}
            \vf(\state,t;\initialvf) &= \min_{\diststrategy \in \diststrategyset} \max_{\ctrlsignal \in \ctrlsignalset} \min\Bigg\{\min_{s \in [t,0)}\constraintfunc\left(\trajectory(s)\right), \initialvf\left(\trajectory(0)\right)\Bigg\}\\
            &\leq \min_{\diststrategy \in \diststrategyset_{[t,0]}} \max_{\ctrlsignal \in \ctrlset_{[t,0]}} \min\Bigg\{\min_{s \in [t,0)}\constraintfunc\left(\trajectory(s)\right), \constraintfunc\left(\trajectory(0)\right)\Bigg\} \\
            &=\min_{\diststrategy \in \diststrategyset_{[t,0]}} \max_{\ctrlsignal \in \ctrlset_{[t,0]}} \min_{s \in [t,0]}\constraintfunc\left(\trajectory(s)\right)\\
            &=\vf(\state,t;\constraintfunc),
        \end{align*}
        where we drop $_{\state,t}^{\ctrlsignal, \diststrategy}$ from the trajectory $\trajectory(s)$ for readability.
        The inequality follows directly from Assumption~\ref{ass:conservative_guess}, i.e. $\initialvf(\state) \leq \constraintfunc(\state)$ for all $\state$. Therefore, as $t\to -\infty$, we have $\converged{\vf}(\state;\initialvf)\leq \converged{\vf}(\state;\constraintfunc)$.
        Combining (a) and (b) we have $\converged{\safeset} = \{\state \mid \converged{\vf}(\state;\initialvf) \geq 0\}\subseteq \{\state \mid \converged{\vf}(\state;\constraintfunc) \geq 0\} = \viability{\constraintset}$, with the last equality by Definition~\ref{def:viability}, concluding the proof.
    \end{proof}
\end{mylem}

\begin{mylem}[Convergence to a CBVF on a subset of the state space]\label{lem:cbvf_refinecbf}
If~\eqref{eq:refinecbf_trajectory} converges to $\converged{\vf}(\state)$ as $t \to -\infty$, i.e. $\converged{\vf}(\state)$ characterizes a stationary solution of~\eqref{eq:vi_continuous}, then $\converged{\vf}(\state)$ is a valid CBVF for all $\state \in \converged{\safeset}$.
\begin{proof}
    Similarly to Lemma~\ref{lem:CI_refineCBF}, we begin by inspecting the stationary version of~\eqref{eq:vi_continuous},~\eqref{eq:stationary_vi}. 
    The second term of the minimum operator implies that any $\converged{\vf}(\state)$ satisfying~\eqref{eq:stationary_vi} satisfies $\lieopt{\converged{\vf}}(\state)\geq 0$ for all~$\state$. 
    
    The value function $\converged{\vf}$, for any choice of $\gi{} \geq 0$ in~\eqref{eq:cbf}, with $\alpha(\converged{\vf}(\state))=\gi{}\converged{\vf}(\state)$ satisfies:
    \begin{equation*}
        \lieopt{\converged{\vf}}(\state) \geq 0 \geq -\gi{}\converged{\vf}(\state),
    \end{equation*}
    for all $\state\in\converged{\safeset}$, with the first inequality by satisfaction of~\eqref{eq:stationary_vi} and the second inequality from $\converged{\vf}(\state)\geq0$ for all $\state\in\converged{\safeset}$.
\end{proof}
\end{mylem}

We are now ready to prove Theorem~\ref{thm:valid_CBVF}:

    \begin{proof}[Proof of Thm~\ref{thm:valid_CBVF}]
        Safety is guaranteed by $\converged{\vf}(\state)$'s 0-superlevel set, $\converged{\safeset}$ being (a) control invariant, Lem.~\ref{lem:CI_refineCBF}, and (b) being a subset of the viability kernel, Lem.~\ref{lem:viability_subset_refineCBF}. 
        Combined with Lem.~\ref{lem:cbvf_refinecbf}, we obtain a safe CBVF, concluding the proof.
    \end{proof}

The convergence guarantees match those of standard HJR, and similarly do not provide guarantees on the rate of convergence to a safe CBVF, but progressively improve the CBVF’s minimum finite-time safety as $t$ increases (Remark~\ref{rem:safety_duration}).
Additionally, the key concern in the online setting is not the final converged value, but the properties of the value function during refinement. 
A crucial property is that once the value function becomes control invariant at any iteration, it remains so thereafter (Lemma~\ref{lem:retain_safety}).

\begin{myremark}[Duration of safety]\label{rem:safety_duration}
    The temporal input of the value function $\vf(\state,t)$ implicitly provides the duration of safety.
    This follows directly from~\eqref{eq:refinecbf_trajectory}, which implies that for any trajectory starting at $\state$ such that $\vf(\state,t)\geq 0$ there exists a control signal $\ctrlsignal\in\ctrlsignalset$ such that for any non-anticipative disturbance strategy $\diststrategy \in \diststrategyset$ the trajectory $\trajectory$ remains in $\constraintset$ for at least time $t$, i.e. $\constraintfunc(\trajectory(s))\geq 0$ for all $s\in[t,0]$.
    Any control signal satisfying the time-varying control set defined in~\eqref{eq:cbvf_online_timevarying} retains safety for at least time $t$, thus acting as a lower bound on the duration of safety.
\end{myremark}

We further note that if the initial approximate CBF $\initialvf(\state)$ is control invariant or if, by coincidence or by deliberate action, there is a $t_1$ such that $\vf(\state,t_1)$ is control invariant, we provide a guarantee on preserving safety while converging:
\begin{mylem}[Refining a safe CBVF will preserve safety]\label{lem:retain_safety} If there exists $t_1\leq 0$ such that $\vf(\state,t_1)$ from~\eqref{eq:refinecbf_trajectory} is control invariant, $\vf(\state,t)$ is control invariant for all $t<t_1\leq 0$.
\begin{proof}
    For any $t\leq t_1 \leq 0$, we can write the following:
    \begin{align*}
        \vf(\state, t)&=\min_{\diststrategy \in \diststrategyset} \max_{\ctrlsignal \in \ctrlsignalset} \min\left\{\min_{s \in [t,0)}\constraintfunc\left(\trajectory(s)\right), \iter{\vf}{0}\left(\trajectory(0)\right)\right\} \\
        &=\min_{\diststrategy \in \diststrategyset} \max_{\ctrlsignal \in \ctrlsignalset} \min\left\{\min_{s \in [t,t_1)}\constraintfunc\left(\trajectory(s)\right), \vf\left(\trajectory(t_1), t_1)\right)\right\}.
    \end{align*}

Then, any $\state\in\safeset(t)$ satisfies (a) $\min_{s \in [t,t_1)}\constraintfunc\left(\trajectory(s)\right)\geq 0$ and (b) $\vf\left(\trajectory(t_1), t_1\right)\geq 0$.
Hence, we can find a trajectory starting at $\state$ that (a) stays safe for at least $t-t_1$ time and (b) enters a control invariant set $\safeset(t_1)$ within $t - t_1$ time, in which it can stay indefinitely.
By construction, $\safeset(t)\subseteq \constraintset$ for all $t_1\leq 0$, hence $\converged{\safeset}$ also characterizes safety.
\end{proof}
\end{mylem}
We leverage Lemma~\ref{lem:retain_safety} to present a conservative version of the \refineCBF algorithm, \saferefineCBF, Alg.~\ref{alg:refineCBF_algorithm_safe}, which first forces contraction to a safe (yet not necessarily performant) CBVF, followed by a possible set expansion.

\subsection{\refineCBF algorithm overview}
Algorithm~\ref{alg:refineCBF_algorithm} presents our numerical implementation for refining a candidate CBF (the in-the-loop extension is highlighted in teal). 
The core of the algorithm is the iterative application of the HJI-VI update rule from~\ref{eq:vi_dp} (line~\ref{ln: updateCBF}), initialized with $\initialvf$.
This process can be run offline until convergence or, more importantly for our work, executed continuously in-the-loop.
In the online setting, the algorithm observes real-time changes to the environment-new obstacles (through~$\constraintset$), modified actuation limits of the robot (through~$\ctrlset$), or changed disturbances to the system (through~$\distset$) (lines \ref{ln: observe}-\ref{ln: updateEnvironment})-into the refinement process. 
The most current value function is continuously published (line \ref{ln: publishCBF}) for use by a safety filter.
As is standard for HJR-based methods, our implementation solves these updates over a discretized state-space grid.

The accompanying safety filter is detailed in Alg.~\ref{alg:safety_filter}, with the in-the-loop extension in teal.
For each nominal control, the filter queries the latest value function (and its gradient) to solve the QP~\eqref{eq:online-cbf-qp-full}. 
This QP finds the minimum-norm control input that satisfies the safety constraint, and this safe action is then applied to the robot (lines \ref{ln:query_cbf}-\ref{ln:apply_control}).
As we only keep track of the most recent computed $\vf(\state)$, we assume $D_t \vf =0$, i.e. the safe control set is characterized by~\eqref{eq:cbvf_online}.

\begin{algorithm}[t]
\caption{\refineCBF {\color{onlinegray}{ (In-the-loop)}}}
\label{alg:refineCBF_algorithm}
\begin{algorithmic}[1]
\REQUIRE $\initialvf(\state):\stateset \mapsto \R:$ Initial value function \\
~~~~~~~~$\constraintfunc(\state): \stateset \mapsto \R:$ Current constraint function
\STATE $k \gets 0$ 
\STATE $\iter{\vf}{0}(\state) \gets \min\{\initialvf(\state), \constraintfunc(\state)\}$
\WHILE{$|\iter{\vf}{k}(\state)-\iter{\vf}{k-1}(\state)| > \epsilon$ {\color{onlinegray} {(or Robot operating)}}}
    \STATE $\iter{\vf}{k+1}(\state) \gets \min\left\{\constraintfunc(\state), \iter{\vf}{k}(\state) + \iter{\Delta}{k}\lieopt \iter{\vf}{k}(\state)\right\}$ \label{ln: updateCBF}\\ 
    \emph{\color{mygray} // Environment updates can occur at a different rate} \\
    \color{onlinegray}\STATE Observe environment \label{ln: observe}
    \STATE $\iter{\mathcal{U}}{k+1}, \iter{\mathcal{D}}{k+1}, \iter{\constraintfunc}{k+1} \gets$ Update $(\iter{\mathcal{U}}{k}, \iter{\mathcal{D}}{k}, \iter{\constraintfunc}{k})$ \label{ln: updateEnvironment}\\
    \STATE \textbf{publishCurrentCBF}($\iter{\vf}{k+1})$ \label{ln: publishCBF}
    \color{black} \STATE $k \gets k+1$
\ENDWHILE
\RETURN $\converged{\vf}(\state)=\iter{\vf}{k}(\state)$~ ~$\color{mygray}\vartriangleleft$\emph{\color{mygray} ~Converged value function}
\end{algorithmic}
\end{algorithm}

\begin{algorithm}[t]
\caption{\refineCBF safety filter {\color{onlinegray}{(In-the-loop)}}}
\label{alg:safety_filter}
\begin{algorithmic}[1]
\STATE $j \gets 0$

\WHILE{Robot operating}
    \STATE $\iter{\state}{j} \gets \mathbf{estimateState}()$ \label{ln:state_estimate}\\ 
    \STATE $\hat{\ctrl} \gets \mathbf{nominalController}()$ \label{ln:nominal_control}\\
    \STATE $\vf^{\iter{\state}{j}}, \frac{\partial{\vf}}{{\partial{\state}}}^{\iter{\state}{j}} \gets \mathbf{getAndQueryCurrentCBF}(\iter{\state}{j})$ \label{ln:query_cbf}\\
    \STATE $\iter{\ctrl}{j} \gets \textbf{solve~\eqref{eq:online-cbf-qp-full}}(\vf[\iter{\state}{j}], \frac{\partial{\vf}}{{\partial{\state}}}[\iter{\state}{j}], \iter{\ctrlset}{j}, \iter{\distset}{j}, \hat{\ctrl})$ \label{ln:solve_cbf}
    \STATE Apply $\iter{\ctrl}{j}$ to robot\label{ln:apply_control}\\
    \emph{\color{mygray} // Environment updates can occur at a different rate} \\
    \color{onlinegray}\STATE Observe environment
    \STATE $\iter{\mathcal{U}}{j+1}, \iter{\mathcal{D}}{j+1}  \gets$ Update $(\iter{\mathcal{U}}{j}, \iter{\mathcal{D}}{j})$ \label{ln:cbf_observe_env} \\
    \color{black} \STATE $j \gets j+1$
\ENDWHILE
\end{algorithmic}
\end{algorithm}

\subsection{Stronger guarantees at increased conservativeness}
While \refineCBF guarantees safety upon convergence (Theorem~\ref{thm:valid_CBVF}), some applications require stricter assurances during the refinement process.
Therefore, we present a modification to Alg.~\ref{alg:refineCBF_algorithm} which guarantees a reduction of false positive states with every iteration:
\begin{mydef}[False positive states]\label{def:false_positive}
    A false positive state is a state $\state$ that a value function $\vf$ characterizes as safe which is not part of the viability kernel for a given dynamical system and a constraint set $\constraintset$, i.e. any $\state$ such that $\state \in \safeset \cap \viability{\constraintset}^C$. 
\end{mydef}

The standard \refineCBF algorithm does not provide this guarantee, as warm-starting with an arbitrary candidate function implies the HJI-VI update~\eqref{eq:vi_dp} is not a contraction mapping, allowing $\safeset$ to expand to include unsafe states. 

To address this, we propose \saferefineCBF (Alg.~\ref{alg:refineCBF_algorithm_safe}), a two-phase algorithm that ensures a monotonic reduction in false positive states:
\begin{enumerate}
    \item \textbf{Retraction phase}: Initially the algorithm uses the contractive HJI-PDE update rule (line~\ref{ln:retraction_eq}). 
    This forces the value function's safe set to shrink until it represents a provably safe, control-invariant subset of the viability kernel. 
    At this point, the set of false positives is empty.
    \item \textbf{Refining phase}: Once a safe set has been established, the algorithm switches to the HJI-VI update rule (line~\ref{ln:expansion_eq}), which is not contractive. 
    This allows the safe set to expand to improve performance, while Lemma~\ref{lem:retain_safety} guarantees that safety, once established, is never lost.
\end{enumerate}

The two-phase implementation in Alg.~\ref{alg:refineCBF_algorithm_safe}, with the in-the-loop extension in teal, guarantees that safety improves with every iteration, as formalized hereafter in Theorem~\ref{thm:refineCBF_algorithm_safe}.

\begin{algorithm}[t]
\captionsetup{labelfont=bf}
\caption{\saferefineCBF {\color{onlinegray}{ (In-the-loop)}}}
\label{alg:refineCBF_algorithm_safe}
\begin{algorithmic}[1]
\REQUIRE $\initialvf(\state):\stateset \mapsto \R:$ Initial value function \\
~~~~~~~~$\constraintfunc(\state): \stateset \mapsto \R:$ Current constraint function
\STATE $k \gets 0$ 
\STATE retracting $\gets \TRUE$ \label{ln:retraction_flag}
\STATE $\initialvf(\state) \gets \min\{\constraintfunc(\state), \initialvf(\state)\}$
\WHILE{$|\iter{\vf}{k}(\state)-\iter{\vf}{k-1}(\state)| > \epsilon$ \OR retracting {\color{onlinegray} {(\OR Robot operating)}}}
    \IF{retracting AND $|\iter{\vf}{k}(\state)-\iter{\vf}{k-1}(\state)| <\epsilon$}
        \STATE retracting $\gets \FALSE$\label{ln:retraction_false}
    \ENDIF

    \IF{retracting}
         \STATE $\iter{\vf}{k+1}(\state) \gets \iter{\vf}{k}(\state) + \iter{\Delta}{k} \min\{0, \lieopt \iter{\vf}{k}(\state)\}$\label{ln:retraction_eq}
    \ELSE
        \STATE $\iter{\vf}{k+1}(\state) \gets \min\left\{\constraintfunc(\state), \iter{\vf}{k}(\state) + \iter{\Delta}{k}\lieopt \iter{\vf}{k}(\state)\right\}$\label{ln:expansion_eq}
    \ENDIF \\
    \emph{\color{mygray} // Environment updates can occur at a different rate} \\
    \color{onlinegray}\STATE Observe environment
    \STATE $\iter{\mathcal{U}}{k+1}, \iter{\mathcal{D}}{k+1}, \iter{\constraintfunc}{k+1} \gets$ Update $(\iter{\mathcal{U}}{k}, \iter{\mathcal{D}}{k}, \iter{\constraintfunc}{k})$ \\
    \STATE \textbf{publishCurrentCBF}($\iter{\vf}{k+1})$
    \IF{$\iter{\ctrlset}{k+1} \subseteq \iter{\ctrlset}{k}$ or $\iter{\distset}{k+1} \supseteq \iter{\distset}{k} $ or $\exists \> \state$ such that $\iter{\vf}{k+1}(\state) > \iter{\constraintfunc}{k+1}(\state)$}\label{ln:condition_retraction}
        \STATE retracting $\gets \TRUE$\label{ln:retraction_true}
        \STATE $\iter{\vf}{k+1}(\state) \gets \min\{\iter{\constraintfunc}{k+1}(\state), \iter{\vf}{k+1}(\state)\}$\label{ln:update_vf_env}
    \ENDIF \\
    
    \color{black} \STATE $k \gets k+1$
\ENDWHILE
\RETURN $\converged{\vf}(\state)=\iter{\vf}{k}(\state)$~ ~$\color{mygray}\vartriangleleft$\emph{\color{mygray} ~Converged value function}
\end{algorithmic}
\end{algorithm}

\begin{mytheorem}[Safer with every iteration]Alg.~\ref{alg:refineCBF_algorithm_safe} monotonically decreases the rate of false positive states with every iteration.\label{thm:refineCBF_algorithm_safe}
\end{mytheorem}

The proof of Theorem~\ref{thm:refineCBF_algorithm_safe} relies on previously introduced Lemmas~\ref{lem:viability_subset_refineCBF} and~\ref{lem:retain_safety}. In addition, we introduce Lemma~\ref{lem:CI_retract} which guarantees the contraction phase converges to a control invariant set to finalize the proof. 

\begin{mylem}[Convergence to a control invariant set]\label{lem:CI_retract} If~\eqref{eq:vf_trajectory} converges to $\converged{\vf}(\state)$ as $t \to -\infty$ for $\constraintfunc(\state)=\initialvf(\state)$, i.e. $\converged{\vf}(\state)$ characterizes a stationary solution of~\eqref{eq:pde_continuous} initialized with $\vf(\state,0)=\initialvf(\state)$, then its associated 0-superlevel set $\converged{\safeset}$ is control invariant, i.e. $\lieopt{\converged{\vf}}(\state)\geq0$ for all $\state\in\converged{\safesetboundary}$.

\begin{proof}
    The stationary version of the continuous PDE CBVF,~\eqref{eq:pde_continuous}, is as follows:
    \begin{equation}\label{eq:stationary_pde}
        \min\bigl\{0, \max_{\ctrl\in \ctrlset}\min_{\dist\in\distset} D_x\vf(\state,t) \cdot F(\state,\ctrl,\dist) \bigr\}=0.
    \end{equation}
    Upon inspection of the second term of the minimum operator, we have the same inequality as in Lemma~\ref{lem:CI_refineCBF}, and hence guarantee control invariance upon convergence.
\end{proof}
\end{mylem}

We are now ready to prove Theorem~\ref{thm:refineCBF_algorithm_safe}.
\begin{proof}[Proof of Thm.~\ref{thm:refineCBF_algorithm_safe}] 
    Alg.~\ref{alg:refineCBF_algorithm_safe} is split into 2 stages, (a) \emph{retracting} and (b) \emph{refining}. We will show that (a) iteratively decreases the number of false positive states to $0$, and that subsequently (b) maintains $0$ false positive states with every iteration.

    For (a), in the \emph{retracting} phase, we use the PDE-based DP update,~\eqref{eq:pde_dp}, in line~\ref{ln:retraction_eq}. 
    By inspection, $\vf(\state,t_1)\leq \vf(\state, t_2)$ for all $\state$ for any $t_1 \leq t_2 \leq 0$ by~\eqref{eq:pde_dp}, thus $\safeset(t_1)\subseteq \safeset(t_2)$.
    This implies, by Def.~\ref{def:false_positive}, that the set of false positive states monotonically decreases with the iterations.

    Next, by Lemma~\ref{lem:CI_retract}, if~\eqref{eq:vf_trajectory} converges to $\converged{\vf}_a(\state), \converged{\safeset}_a$ is a control invariant set (The subscript $a$ denotes stage $a$ of \saferefineCBF).
    Additionally, by Lemma~\ref{lem:viability_subset_refineCBF}, $\converged{\safeset}_a\subseteq \viability{\constraintset}$, so by Definition~\ref{def:false_positive}, post stage (a), $\converged{\safeset}_a(\state)$ has 0 false positive states. 
    
    For stage (b), refining a safe CBVF (here $\converged{\vf}_a$) preserves safety (Lemma~\ref{lem:retain_safety}). 
    In other words, $\iter{\safeset_b}{k}$ has $0$ false positive states for every $k$, concluding the proof. 
\end{proof}

\subsection{Implications for online adaption in evolving environments}\label{subsec:implications_refinecbf}
The guarantees established previously assume a static environment, as is standard in HJR literature~\cite{BansalChenEtAl2017b}. 
We now analyze the practical scenario where either a converged or converging value function faces a sudden environmental change. 
While the refinement process will eventually converge to a new safe value function, the safety filter's transient behavior during this adaptation is critical for safety. 
We categorize the nature of these environmental changes as follows:
\begin{itemize}
    \item \textbf{Safety enhancing updates} (${\ctrlset\uparrow}, {\distset\downarrow}, {\constraintset\uparrow}$): These changes enlarge the true viability kernel, such as increased actuation limits, reduced disturbance bounds, or a decrease in failure states. 
    These updates simply render the current value function to be ``safer'', thus maintaining or improving the degree of safety of the associated filter.
    
    \item \textbf{Safety decreasing updates} ($  {\ctrlset\downarrow}, {\distset\uparrow}, {\constraintset\downarrow}$): These changes shrink the true viability kernel, such as reduced actuation limits, larger disturbances, or new obstacles. 
    In this scenario, the existing value function may now falsely label unsafe states in the new environmental condition as safe. 
    For this scenario, to quantify the degree of safety (through Def.~\ref{def:false_positive}), we compare the first observation after the update ($k+1$) with the prior value function ($k$), with both considering safety with respect to the novel environment. 
\end{itemize}

The two proposed algorithms handle safety-decreasing updates differently. 
\refineCBF has no special mechanism for updates; it continues iterating and eventually (re-)converges to a safe value function for the new environment, but offers no guarantees during this transient period. 
In contrast, \saferefineCBF is explicitly designed to handle this scenario robustly. 
As detailed in Alg.~\ref{alg:refineCBF_algorithm_safe}, at risk safety-decreasing updates force the algorithm back to its retraction phase (line~\ref{ln:retraction_true}).
This ensures the system contracts (if needed) to a control invariant set under the new environment before attempting to expand again. 
This provides a principled way to ensure that the safety improves with each iteration, allowing Thm.~\ref{thm:refineCBF_algorithm_safe} to extend to in-the-loop operation. 

Note that this algorithm is primarily intended for piecewise stationary environments, rather than those with continuously time-varying dynamics.
In addition, the safety of the filter is not assured during safety decreasing updates, see Remark~\ref{rem:safety_filter}.

\begin{myremark}[Safety of the filter]\label{rem:safety_filter}
    The safety increasing result of Alg.~\ref{alg:refineCBF_algorithm_safe} between iterations is defined as a decrease in false positive states. 
    This does not imply that safety of the system is guaranteed when faced with an unexpected safety decreasing update, as the state of the system can be within the newly false positive set. 
    During the retraction phase, to encourage returning to safety, it is possible to either (i) modify the nominal controller's goal (if goal position-conditioned) to the closest state inside the current safe set or (ii) include a $D_t \iter{\vf}{k}(x) \approx \frac{\iter{\vf}{k} - \iter{\vf}{k-1}}{\Delta^{(k)}}$ term to tighten the linear inequality of~\eqref{eq:online-cbf-qp-full}, although this is now not guaranteed to be feasible with the input constraints.
\end{myremark}

\subsection{Use-cases and limitations of \refineCBF}
The dense and uniform nature of refineCBF's global computation is highly amenable to massive parallelization on modern GPUs, offering the potential for significant computational speedups. 
Nonetheless, a key property of the value function $\converged{\vf}(\state)$ obtained by \refineCBF is that it satisfies a stronger condition than is strictly necessary. 
While a standard CBF only requires its 0-superlevel set to be control invariant, the converged HJI-based value function has the property that all of its superlevel sets $\{\state\mid\converged{\vf}(\state)\geq \beta\}$ for all $\beta\in\R$ are control invariant. 
Enforcing this unnecessarily strong condition can increase the computational burden and number of iterations required for convergence. 
In addition, while converging, this method requires updating the entire grid at each iteration.  
This is particularly inefficient in common scenarios where the initial candidate CBF is already ``mostly safe'' and requires only minor, localized corrections rather than a full global refinement. 
This limitation directly motivates \algname, which is designed for precisely this use case.

\section{Patching CBFs}\label{sec:hjpatch}
As a computationally efficient alternative to \refineCBF, we introduce \algname, an algorithm designed to ``patch'' an existing candidate value function, $\initialvf$, rather than performing a full global refinement.  
The core trade-off is this: instead of converging to a control invariant set within the constraint set $\constraintset$ (as \refineCBF does with the VI formulation), \algname finds the largest control invariant subset within the 0-superlevel set of the initial guess, $\initialvf$. 
This is achieved by using the contractive HJI-PDE formulation~\eqref{eq:pde_continuous}. 

The primary source of computational speedup comes from a ``local'' dynamic programming approach. 
We introduce the following notation to denote near-boundary states of the value function $\vf$ and its 0-superlevel set $\safeset$:
\begin{equation*}
    \boundarycells = \{x \mid \lvert h(x) \rvert \leq \zeta\}
\end{equation*}
Our key insight is that for the HJI-PDE formulation, we only need to update states that directly impact the invariance of the approximate safe set, i.e. value decreasing states near the boundary. 
Therefore, \algname maintains an active set of states, which we denote $\activeset$.
At each iteration, only the states in this active set are updated, drastically reducing the computational cost compared to a global update. 
This active set is intelligently truncated to exclude states that are non-contractive (and thus locally safe), and expanded to include neighbors of states whose values decrease; this ensures that the ``patching'' correctly propagates while maintaining computational efficiency. 
The full method is detailed in Algorithm~\ref{alg:HJR_boundary_march}.
\begin{algorithm}[!b]
\caption{\algname {\color{onlinegray}{ (In-the-loop)}}}
\label{alg:HJR_boundary_march}
\begin{algorithmic}[1]
\REQUIRE $\initialvf(\state):\stateset \mapsto \R:$ Initial value function \\
~\>~~~~~~$\constraintfunc(\state):\stateset \mapsto \R:$ Constraint function (Optional) \\
~\>~~~~~~$\oraclecells\subseteq \stateset:$ Oracle-certified safe cells (Optional)\\
\emph{\color{mygray}// Active set is set of potentially unsafe states}
\STATE $\initialvf(\state) \gets \min\{\constraintfunc(\state), \initialvf(\state)\}$
\STATE $\iter{\activeset}{0} \gets \iter{\boundarycells}{0} \label{ln:init_active}
\hspace{1.05cm}\color{mygray}\vartriangleleft$ {\color{mygray}\emph{Initialize active set}}
\STATE $\iter{\activeset}{0} \gets \iter{\activeset}{0} \setminus C $ ~~~~~~~~$\color{mygray}\vartriangleleft$~{\color{mygray}\emph{Remove oracle-certified cells}}\label{ln:remove_oracle}
\STATE $k \gets 0$ 

\emph{\color{mygray}// Boundary is certified once $\iter{\activeset}{k}$ is empty}
\WHILE{$\iter{\activeset}{k}$ is not empty {\color{onlinegray}{(or Robot operating)}}}\label{ln:while_patchhj}
    \FOR{$\state \in \iter{\activeset}{k}$}\label{ln:subset_forloop}
        \STATE $\iter{\vf}{k+1}(\state) \gets \iter{\vf}{k}(\state) + \iter{\Delta}{k} \min\{0, \lieopt \iter{\vf}{k}(\state)\}$\label{ln:update_eq}
    \ENDFOR \\
    \STATE $\iter{\safeset}{k+1} \gets \{\state \mid \iter{\vf}{k+1}(\state) \geq 0\}$ \label{ln:new_safe_set} \\
    \emph{\color{mygray}// Value-decreasing states form set of interest}
    \STATE $\iter{Q}{k+1} \gets \{\state \mid \iter{\vf}{k+1}(\state) \neq \iter{\vf}{k}(\state) \wedge \state \in \iter{\activeset}{k}$ \}\label{ln:intermediate_set}\\
    \emph{\color{mygray}// Pad set with neighbors if near safe set boundary} \label{ln:set_of_interest}
    \STATE $\iter{\activeset}{k+1} \gets \mathbf{pad}^p(\iter{Q}{k+1}) \cap \iter{\boundarycells}{k+1}$ \label{ln:updated_active_set}\\
    \color{onlinegray}\STATE Observe environment\label{ln:observe_environment}
    \STATE $\iter{\mathcal{U}}{k+1}, \iter{\mathcal{D}}{k+1}, \iter{\constraintfunc}{k+1} \gets$ Update $(\iter{\mathcal{U}}{k}, \iter{\mathcal{D}}{k}, \iter{\constraintfunc}{k})$ \label{ln:process_environment}\\
    \emph{\color{mygray} // Less actuation or more disturbance can be unsafe}\\
    \IF{$\iter{\ctrlset}{k+1} \subseteq \iter{\ctrlset}{k}$ or $\iter{\distset}{k+1} \supseteq \iter{\distset}{k}$}\label{ln:new_actuation_dist}
        \STATE $\iter{\activeset}{k+1} \gets \iter{\boundarycells}{k+1}$ $\color{mygray}\vartriangleleft$\emph{\color{mygray} ~All boundary states active}\label{ln:less_power_newactiveset}
    \ENDIF \\
    \emph{\color{mygray} // New obstacles can be unsafe}
    \IF{$\exists \> \state$ such that $\iter{\vf}{k+1}(\state) > \iter{\constraintfunc}{k+1}(\state)$}\label{ln:new_obstacles}
        \STATE $T \gets \{\state \mid \iter{\constraintfunc}{k+1}(\state) < \iter{\vf}{k+1}\}$ $\color{mygray}\vartriangleleft$\emph{\color{mygray} ~New active states}\label{ln:active_set_new_obstacles}
        \STATE $\iter{\vf}{k+1}(\state) \gets \min\{\iter{\constraintfunc}{k+1}(\state), \iter{\vf}{k+1}(\state)\}$ \label{ln:new_localvf}\\
        \STATE $\iter{\safeset}{k+1} \gets \{\state \mid \iter{\vf}{k+1}(\state) \geq 0\}$ \label{ln:new_localset}\\
        \emph{\color{mygray} // New active set includes prior and obstacle-updated active states around new boundary}        
        \STATE $\iter{\activeset}{k+1}  \gets (\iter{\activeset}{k+1} \cup T) \cap \iter{\boundarycells}{k+1}$\label{ln:active_set_full_completion}
    \ENDIF
    \STATE \textbf{publishCurrentCBF}($\iter{\vf}{k+1})$
    \color{black}\STATE $k \gets k+1$
\ENDWHILE
\RETURN $\converged{\vf}(\state)=\iter{\vf}{k}(\state)$~ ~$\color{mygray}\vartriangleleft$\emph{\color{mygray} ~Converged value function}\label{ln:return_hjpatch}
\end{algorithmic}
\end{algorithm}

Upon convergence, \algname provably recovers the same 0-superlevel set as a global contraction - the viability kernel of the initial set, $\viability{\iter{\safeset}{0}}$, see Theorem~\ref{thm:algsafe}. 
Additionally, we recover a safe CBVF upon convergence for an appropriately chosen $\gamma$ in~\eqref{eq:online-cbf-qp-full}.  
The quality of the converged value function is directly correlated to the quality and conservativeness of the initial candidate value function $\initialvf$.

\subsection{Algorithm details}
We discuss \algname for offline convergence and online adaptation below.
 
\paragraph{Offline patching} Alg.~\ref{alg:HJR_boundary_march} takes an approximate CBF $\initialvf(\state)$, optionally a constraint function $\constraintfunc(\state)$, and optionally a set of oracle-certified safe states as its input. 
As detailed in Algorithm~\ref{alg:HJR_boundary_march}, the iterative process is initialized with a candidate value function $\initialvf(\state)$. 
The initial active set, $\iter{\activeset}{0}$, comprises the states near the boundary of the candidate safe set (line~\ref{ln:init_active}). 
If an external oracle can certify that a subset of states $C$ is already safe (e.g., via neural network verification techniques~\cite{LiuArnonEtAl2021,FazlyabRobeyEtAl2019}), these states can be removed from the active set for further computational efficiency (line~\ref{ln:remove_oracle}). 
The iterative loop applies the contractive HJI-PDE update~\eqref{eq:pde_dp} to states in the active set (lines~\ref{ln:subset_forloop}-~\ref{ln:update_eq}). 
The active set is then updated to filter out locally safe states, i.e. non-decreasing values (line~\ref{ln:intermediate_set}) after which it is expanded to include its neighbors (line~\ref{ln:updated_active_set}). 
$\mathbf{pad}^p$ denotes padding a space-discretized set with its $p$ neighbors in every dimension (equivalent to the Minkowski sum with a ball of small radius in continuous space): 
\begin{equation}\label{eq:padding}
    \mathbf{pad}^p(\activeset) = \bigcup_{r \in \activeset} \neighbors{p}(r).
\end{equation}
At every iteration $k$, $\iter{\activeset}{k}\subseteq\iter{\boundarycells}{k}$ by line~\ref{ln:updated_active_set}, ensuring only a subset of the boundary cells are ``active''. 
This procedure repeats until the active set is empty (line~\ref{ln:return_hjpatch}) indicating convergence.

\paragraph{Online patching} 
When used in-the-loop, the algorithm includes additional logic to handle environmental changes. 
Safety-decreasing updates require special care to ensure that states previously assumed safe are re-evaluated. 
If actuation limits decrease or disturbance bounds increase, the algorithm conservatively resets the active set to the current boundary states to force a full re-evaluation (line~\ref{ln:less_power_newactiveset}). 
If new obstacles appear, the value function is first clipped to respect the new constraints, and the active set is expanded to include any boundary states that coincide with the new obstacles’ boundaries (lines~\ref{ln:new_obstacles}-~\ref{ln:active_set_full_completion}). 
Safety-enhancing updates do not require changing the active set.

\subsection{Theoretical guarantees for \algname}
Because \algname updates only a subset of states at each iteration, our theoretical analysis must explicitly consider the discretized state-space on which the algorithm operates. 
These guarantees assume a static environment. 
Theorem~\ref{thm:algsafe} provides the key theoretical underpinning of \algname.

\begin{mytheorem}[Algorithm~\ref{alg:HJR_boundary_march} recovers the viability kernel of the initial candidate safe set]\label{thm:algsafe}
    If Alg.~\ref{alg:HJR_boundary_march} converges, then, initializing with $\iter{\vf}{0}(\state)=\vf^0(\state)$, we have $\converged{\safeset} = \viability{\safeset^0}$. 
\end{mytheorem}
The proof of Theorem~\ref{thm:algsafe} relies on three supporting lemmas. 
We first define a spatially discretized equivalent of control invariance, which we term \quasisynonym control invariance. 
We then prove that \algname converges to a quasisynonym control invariant set (Lemma~\ref{lem:CI_set}), and that, leveraging Lemma~\ref{lem:optimistic}, this set is a superset of the viability kernel (Lemma~\ref{lem:superset}). 
Together, these properties are sufficient to prove the main theorem. 

Inspired by Nagumo's Theorem~\eqref{eq:nagumo}, we first define a \quasisynonym control invariant set for a discretized state space.
\begin{mydef}[\Quasisynonym control invariant set $\safeset$]\label{def:quasiCI}
Let $\vf(\state)$ be defined on a discretized state-space and $\safeset$ its 0-superlevel set. Assuming $\frac{\partial \vf}{\partial \state}\neq 0$ for all $\state \in \boundarycells$, then $\safeset$ is \quasisynonym control invariant if and only if
\begin{equation}\label{eq:discreteCI}
    \lieopt\vf(\state) \geq 0 \text{ for all } \state \in \boundarycells.
\end{equation}
\end{mydef}

While~\eqref{eq:discreteCI} is spatially more restrictive than~\eqref{eq:nagumo}, it acts as a conservative buffer against interpolation errors that otherwise preclude a strict guarantee on coarse grids or for value functions with high Lipschitz constants.

\begin{mylem}[Algorithm~\ref{alg:HJR_boundary_march} converges to a control-invariant set]\label{lem:CI_set}
    If Alg.~\ref{alg:HJR_boundary_march} converges, then the safe set $\converged{\safeset}$ obtained upon termination of the algorithm is \quasisynonym control invariant.
\begin{proof}
    Without loss of generality, we consider $\oraclecells=\emptyset$, i.e. no cells are oracle-certified upon initiation.
       Recall by Def.~\ref{def:quasiCI} that $\safeset$ is \quasisynonym control invariant if and only if $\lieopt{\vf}(\state)\geq 0$ for all $\state\in\boundarycells$. 
   Specifically, we seek to show that for every iteration $k$, every boundary cell not in the current active set has a positive Lie derivative, i.e. if $\state \in \iter{\boundarycells}{k} \setminus \iter{\activeset}{k}$, then $\lieopt\iter{\vf}{k}(\state) \geq 0$. Then, if Alg.~\ref{alg:HJR_boundary_march} converges, i.e. $\activeset = \emptyset$, we either have (i) $\boundarycells = \emptyset$ implying $\safeset = \emptyset$ or (ii) $\lieopt\vf(\state) \geq 0$ for all $\state \in \boundarycells$.
    
    We proceed by induction. 
    Recall that $\iter{\activeset}{0}=\iter{\boundarycells}{0} \setminus C$. 
    We have to certify that if $\state \in \iter{\boundarycells}{0} \setminus \iter{\activeset}{0}=\iter{\boundarycells}{0} \cap C$, then $ \lieopt{\iter{\vf}{0}}(\state) \geq 0$, which is guaranteed for all $\state\in \oraclecells$. 
    
    Next, for iteration $k$ we assume that if $\state \in \iter{\boundarycells}{k} \setminus \iter{\activeset}{k}$, then $\lieopt \iter{\vf}{k}(\state) \geq 0$. We hence need to show that $\state \in \iter{\boundarycells}{k+1} \setminus \iter{\activeset}{k+1} $ implies $\lieopt \iter{\vf}{k+1}(\state) \geq 0$, or its contrapositive; if $\lieopt \iter{\vf}{k+1}(\state) < 0$, then $\state \in \iter{\activeset}{k+1}$ for all $\state \in \iter{\boundarycells}{k+1}$. For all $\state\in \iter{\boundarycells}{k+1}$ it is then sufficient to show $\state\in\iter{\activeset}{k+1}$ for the following scenarios: (i) $\lieopt \iter{\vf}{k+1}(\state)=\lieopt\iter{\vf}{k}(\state) < 0$ and (ii) $\lieopt \iter{\vf}{k+1}(\state) \neq \lieopt \iter{h}{k}(\state)$.

    For (i), by assumption, if $\lieopt\iter{\vf}{k}(\state)<0$, then $\state \in \iter{\activeset}{k}$. As $\lieopt\iter{\vf}{k}(\state) < 0$, by~\eqref{eq:pde_dp}, $\iter{\vf}{k+1}(\state)\neq\iter{\vf}{k}(\state)$, hence $\state\in \iter{Q}{k+1}$, and by extension $\state \in \iter{\activeset}{k+1}$ if $\state\in\iter{\boundarycells}{k+1}$.
    
    For (ii), by~\eqref{eq:finite_difference}, $\lieopt \iter{\vf}{k+1}(\state) \neq \lieopt \iter{h}{k}(\state)$ implies that there exists a state $y \in \neighbors{p}$ such that $\iter{\vf}{k+1}(y) \neq \iter{h}{k}(y)$. 
    Hence, by Alg.~\ref{alg:HJR_boundary_march} (L\ref{ln:intermediate_set}), $y \in \iter{Q}{k+1}$. By definition, $\state$ is a neighbor of $y$ (bijective mapping), hence $\state \in \mathbf{pad}^p(\iter{Q}{k+1})$. Then, $\state \in \iter{\activeset}{k+1}$ if $\state \in \iter{\boundarycells}{k+1}$ (Alg.~\ref{alg:HJR_boundary_march} (L\ref{ln:updated_active_set})).
\end{proof}
\end{mylem}

We also point out the following:
\begin{myremark}
    It is possible that positive-valued states at an iteration $k$ (or oracle-safe cells at iteration $k=0$) become part of the active set at a later iteration through padding. 
\end{myremark}

\begin{mylem}[{{Optimistic global warm-start HJ reachability recovers the viability kernel~\cite[Thm. 2]{HerbertBansalEtAl2019}}}]\label{lem:optimistic} 
Let $\vf^0(\state)$ be the initial function and $\converged{\vf}(\state)$ be the value obtained upon convergence of~\eqref{eq:pde_continuous}. Then, there exists a continuous function $m:\stateset\mapsto\R$ such that $\viability{\safeset^0}=\{\state \mid m(\state)\geq0\}$, $h^0(\state)\geq m(\state)$ for all $\state\in\stateset$, and $\converged{\vf}(\state)=m(\state)$ for all $\state\in\stateset$.
\end{mylem}

\begin{mylem}\label{lem:superset}
    [Algorithm~\ref{alg:HJR_boundary_march}'s safe set upon convergence is a superset of the viability kernel] 
    If Alg.~\ref{alg:HJR_boundary_march} converges, then, initializing with $\iter{\vf}{0}(\state) = \vf^0(\state)$, we have $\converged{\safeset} \supseteq \viability{\safeset^0}$ upon convergence.
    
    \begin{proof}
        We denote $\Lambda$ as the single-step (i.e. within the while loop) operator for Alg.~\ref{alg:HJR_boundary_march} and $\Gamma$ as the single-step operator for standard reachability, i.e.~\eqref{eq:pde_dp}. 
        It applies to both $\vf(\state)$ and $\safeset$. 
        Hence, $\iter{\safeset}{k+1}=\Lambda (\iter{\safeset}{k})$. 
        Lastly, we note that $\vf(\state) \geq g(\state)$ for all $\state$ if and only if $\safeset \supseteq \mathcal{G}$, with $\safeset$ and $\mathcal{G}$ the 0-superlevel sets of $\vf(\state)$ and $g(\state)$.
        
        Then, if for every iteration $k$, $\iter{\safeset}{k} \supseteq \viability{\safeset^0}$, we converge to a superset of $\viability{\safeset^0}$. 
        We proceed by induction. 
        By definition of the viability kernel (Def.~\ref{def:viability}), we have $\iter{\safeset}{0} \supseteq \viability{\safeset^0}$. 
        
        Next, for iteration $k$ we assume $\iter{\vf}{k}(\state)$ is such that $\iter{\safeset}{k} \supseteq \viability{\safeset^0}$. 
        We hence need to show that $\iter{\safeset}{k+1} \supseteq \viability{\safeset^0}$. 
        Particularly, we will show that $\iter{\safeset}{k+1} \supseteq \mathcal{G}$, and $\mathcal{G}\supseteq \viability{\safeset^0}$, for $\mathcal{G}=\Gamma (\iter{\safeset}{k})$.
    
        By construction, for any value function $\vf(\state)$, by Alg.~\ref{alg:HJR_boundary_march}, we have $\Lambda(\vf(\state)) \geq \Gamma(\vf(\state))$, hence $\Lambda(\safeset) \supseteq \Gamma(\safeset)$. 
        Thus, $\iter{\safeset}{k+1} =\Lambda(\iter{\safeset}{k}) \supseteq \Gamma(\iter{\safeset}{k}) = \mathcal{G}$. 
        
        It remains to show that $\mathcal{G}$ is such that $\mathcal{G} \supseteq \viability{\safeset^0}$. 
        By Lemma~\ref{lem:optimistic} and given $\iter{\safeset}{k} \supseteq \viability{\safeset^0}$, applying standard reachability~\eqref{eq:pde_dp} recursively to $\iter{\vf}{k}(\state)$ recovers a value function $\hat{\vf}^*(\state)=\Gamma \circ \cdots \circ \Gamma(\iter{\vf}{k}(\state))$, such that $\hat{\safeset}^*=\viability{\safeset^0}$. 
        Noting that $\Gamma$ is a contraction mapping, we have $\viability{\safeset^0} = \Gamma \circ \cdots \circ \Gamma (\iter{\safeset}{k}) \subseteq \Gamma(\iter{\safeset}{k}) =\mathcal{G}$. 
    
        Combining, we have $\iter{\safeset}{k+1} \supseteq \viability{\safeset^0} $, concluding the proof.
    \end{proof}
\end{mylem}

We are now ready to prove Theorem~\ref{thm:algsafe}
\begin{proof}[Proof of Thm.~\ref{thm:algsafe}]
    Combining Lemma~\ref{lem:CI_set} and Lemma~\ref{lem:superset} and noting that the viability kernel is defined as a set's largest control invariant subset (Def.~\ref{def:viability}), we directly obtain $\converged{\safeset} = \viability{\safeset^0}$.
\end{proof}

Furthermore, because the algorithm relies on a contractive update, it provides the same strong increasing safety guarantee as \saferefineCBF. 
\begin{mytheorem}[Safer with every iteration]
    Alg.~\ref{alg:HJR_boundary_march} montonically decreases the rate of false positive states with every iteration
\begin{proof}
    \algname guarantees contraction of the value function $\vf$ with every iteration, thus guaranteeing a monotonic decrease of the number of false positive states (Def.~\ref{def:false_positive})
\end{proof}
\end{mytheorem}

\subsection{Implications and use-cases for \algname}
The theoretical results highlight the fundamental difference between our two proposed algorithms. 
\refineCBF can expand the safe set beyond the initial guess, making it ideal for improving overly conservative initializations, including analytical and backup CBFs. 
In contrast, \algname is designed to efficiently patch an initial guess by finding its largest control invariant subset. 
This makes it particularly well-suited for patching ``almost-safe’’ learned CBFs, which may be performant but contain small, critical failure regions. 
One core difference between the theory outlined for \refineCBF and \algname is the difference in the obtained safe set. 
Additionally, unlike \algname, \refineCBF can reduce conservatism after safety-increasing updates by expanding the safe set. 

The primary benefit of \algname is its computational performance. 
However, the speedup is most significant when the ``patch’’ needed is very localized (relative to the entire state space). 
The algorithm's local update structure is inherently well-suited for modern multi-core CPUs, which excel at the logic required to process a small, irregular active set. 
Our CPU-based implementation leverages this for significant performance gains, especially for localized active sets.  
A GPU implementation presents a more nuanced trade-off: while a naive approach would suffer from poor hardware utilization, high performance is theoretically achievable using advanced patterns like stream compaction to first gather the active states. 
This, however, adds a non-trivial pre-processing overhead. 
Given the demonstrated effectiveness of the CPU for patching localized errors, we leave a detailed GPU implementation and performance comparison as an avenue for future work.

\section{Experiments (simulation)}\label{sec:sim} 
Previous work established the efficacy of \refineCBF~\cite{TonkensHerbert2022} and its computational scalable successor, \algname~\cite{TonkensToofanianEtAl2024}, for offline value function refinement. 
We demonstrated that these methods can improve conservative, analytically-derived CBFs for adaptive cruise control, correct unsafe and relax conservative CBFs for high-dimensional quadcopter dynamics, and scale to patch neural network-based CBFs. 
However, these validations were conducted offline on static, pre-defined environments. 

The central contribution of this paper is to evaluate \refineCBF and \algname for in-the-loop adaptation to non-stationary environments. 
Here, the algorithms must update the safety-critical value function in real-time in response to environmental changes, such as the appearance of new obstacles or shifts in system dynamics.

A key challenge for in-the-loop deployment is computational cost. 
However, a principal advantage of our methods is their ability to warm-start the refinement process from a previously computed value function, a prior CBF from a similar environment, or even a simple signed-distance function. 
This warm-starting capability is crucial for rapid adaptation. 
While formal safety guarantees are contingent upon the convergence of the algorithms, \saferefineCBF and \algname theoretically ensure that the value function's safety improves with every iteration. 
In practice, \refineCBF has almost monotonic improvement in our experiments. 
Many realistic environmental pertubations, such as those in our experiments, are ``localized'' in their impact. 
This allows the algorithms to rapidly refine or patch the value function and consistently converge within a few iterations. 

To validate this online adaptation capability, we conduct a series of experiments in realistic simulators, comparing our methods against relevant baselines. Each experiment is repeated over $10$ runs ($5$ if all failures) to ensure statistical significance. 
All simulations were performed on a desktop computer equipped with an NVIDIA RTX 3090 GPU and 32GB of RAM. 
All controllers run at $50$Hz with the value function updates running in parallel with $\iter{\Delta}{k}=0.1$s for all iterations $k$ at a typical rate of $1-3$Hz.

\subsection{Ground robot: Combining obstacle-dependent CBFs} 
For systems like ground vehicles, CBFs for individual obstacles can be derived analytically or via Hamilton-Jacobi (HJ) reachability. 
A common practice is to enforce these individual CBFs as independent constraints in an optimization problem. 
However, these constraints are not necessarily jointly feasible, which can compromise safety.

\begin{figure}[t]

    \centering
    \includegraphics[width=\columnwidth]{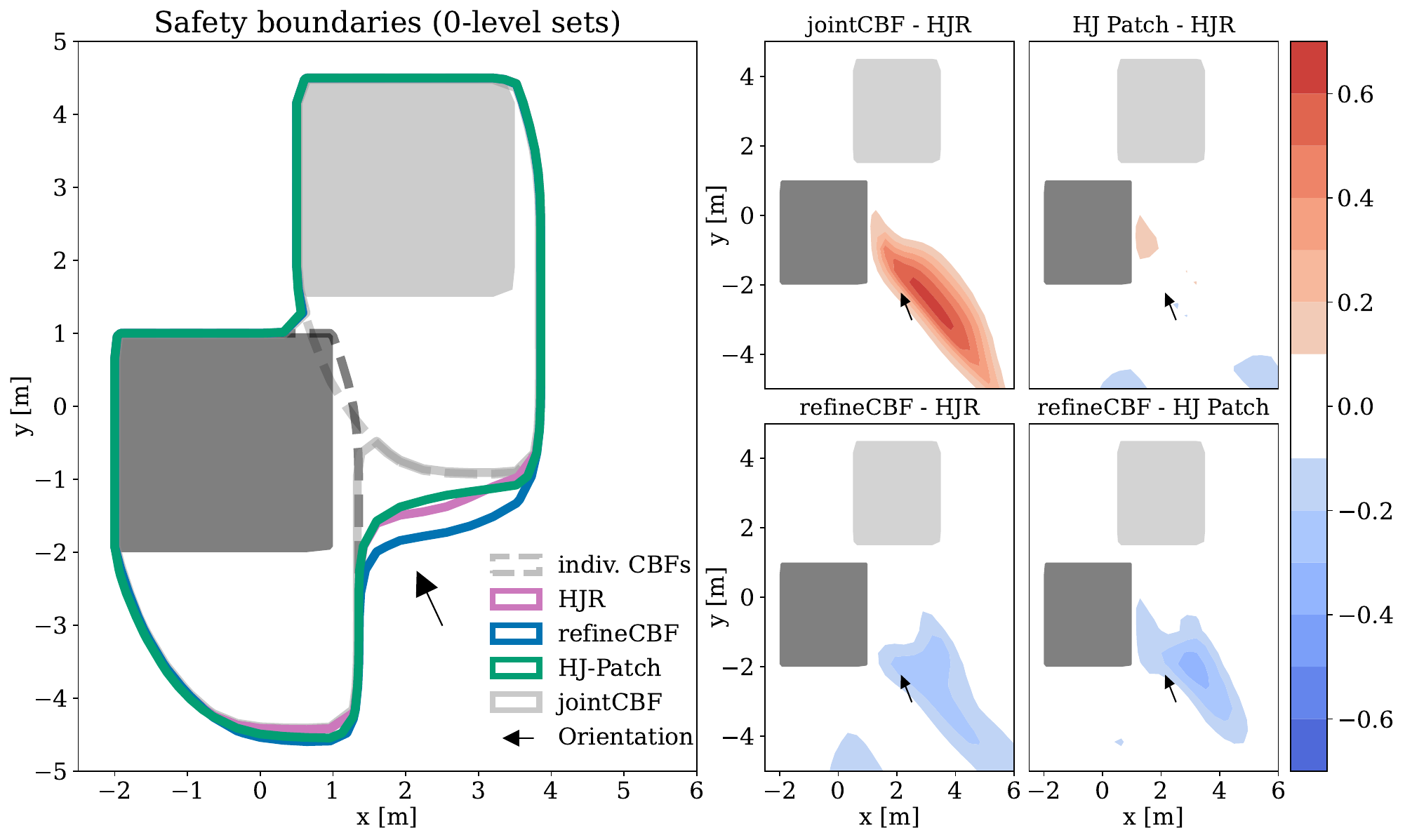}
    \captionsetup{labelfont={bf}}
    \caption{Comparison of converged value functions for \refineCBF, \algname, HJR, and jointCBF at a fixed robot orientation (see arrow) and velocity ($v=1.0$m/s) slice. 
    (Left) The 0-level sets show that naively combining constraints (jointCBF) incorrectly merges the safe sets, creating an unsafe region between the obstacles. \algname and \refineCBF both produce safe, comparable boundaries compared to the HJR baseline.
    (Right) The difference plots highlight that \algname's value function is a close match to the HJR solution, while \refineCBF is more conservative (negative differences) and jointCBF is dangerously optimistic (large positive differences).}
    \label{fig:comparison_levelsets}
\vspace{-0.5cm}
\end{figure}

Our experiment investigates this problem in a Gazebo simulation where a Jackal robot, modeled as a dynamically extended unicycle model with a limited-range LiDAR navigates to a goal. We model the jackal robot as a dynamically extended bicycle model:
\begin{align}\label{eq:extended_unicycle}
    \dot{p}_x = v\cos(\theta), \>\>\>\>\> \dot{p}_y = v \sin (\theta),\>\>\>\>\>\dot{v} = a, \>\>\>\>\> \dot{\theta} = \omega,
\end{align}
with inputs $[a, \omega]$.
Upon detecting a new obstacle, we must update the safety controller online. 
We compare our warm-started methods—\refineCBF (on both CPU and GPU) and \algname (CPU)—against two baselines: naively combining constraints (jointCBFs) and re-solving the problem from scratch using HJ reachability (HJR). 
All methods operate under a minimum forward velocity constraint of 0.1 m/s to prevent deadlock and to showcase how our methods systematically incorporates bounded control inputs. 

First, we analyze the quality of the converged value functions (if given enough time and perfect knowledge of the environment) in a static two-obstacle environment (Fig.~\ref{fig:comparison_levelsets}), where \refineCBF and \algname are initialized from the individual CBF of each obstacle. 
The individual CBFs, $\vf_i$'s, are computed by running~\eqref{eq:pde_dp} until convergence for the rectangular single obstacle constraint function $\constraintfunc(x)=\constraintfunc(p_x,p_y)$, thus $\vf_i(\state)=\vf_i(p_x,p_y,v,\omega)$. 
Naively combining individual CBFs $\vf(\state)=min_i\{\vf_i(\state)\}$ (jointCBF) produces an overly optimistic and unsafe value function. 
In contrast, \algname successfully computes a safe value function that closely approximates the ground truth from HJR, while \refineCBF is safe but more conservative.

Next, we evaluate the critical challenge of online adaptation as the robot discovers obstacles in real time (Fig.~\ref{fig:jackal_sim_online}). 
The results, quantified in Table~\ref{tab:jackal_comp}, demonstrate that the baseline approaches lead to collisions for distinct reasons. 
The jointCBF method is inherently unsafe, as it fails to account for joint feasibility between obstacle constraints. 
In parallel, re-computing the ground-truth safe set from scratch (HJR) is computationally intractable for real-time updates. 
In contrast, while \refineCBF requires a GPU to be fast enough for safe online updates, \algname provides robust safety performance on a CPU in 9/10 rollouts.
A single outlier is visible in Fig.~\ref{fig:jackal_sim_online} (right). The logs show a communication delay between computation of $\iter{\vf}{k}$ and its transmission to the safety filter at 2 subsequent iterations, which could explain the failure. 
This highlights \algname's computational efficiency, making it a viable solution for resource-constrained platforms where GPU acceleration is unavailable. 

\begin{figure}[t]

    \centering
    \includegraphics[width=\linewidth]{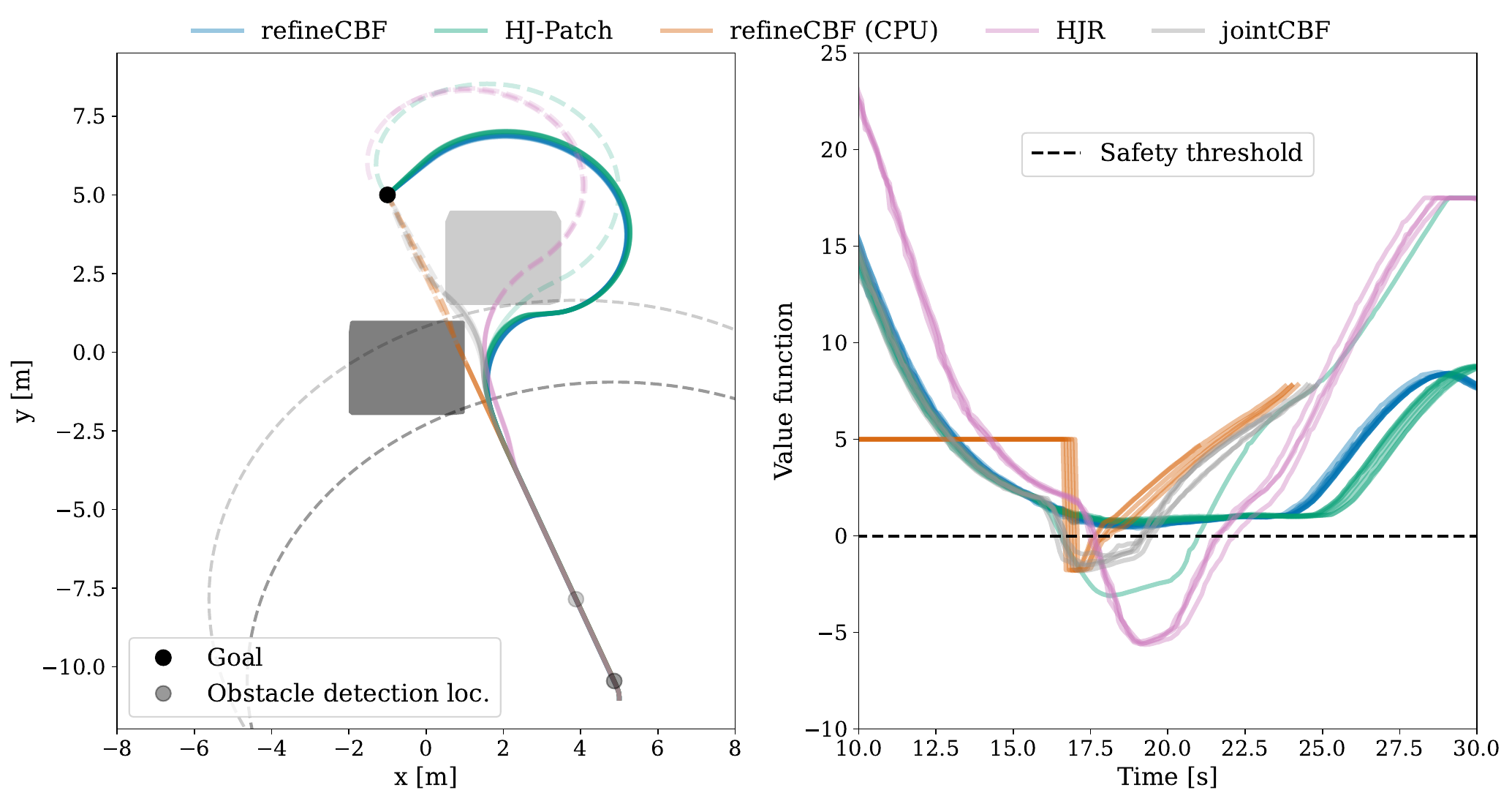}
    \caption{Online performance with limited-range obstacle detection. These plots show system behavior across 10 simulation rollouts where obstacles are discovered online. 
    (Left) shows the trajectories for all methods (with dotted lines for trajectories post collision) for multiple rollouts. 
    We additionally visualize the detection points based on the limited range of the sensor. 
    (Right) visualizes the value function (with the safety boundary). 
    These rollouts demonstrate the need for refinement, as HJR and jointCBF lead to collision. 
    Additionally, when a GPU is not available, \algname's computational efficiency enables patching rapidly to provide safety, whereas \refineCBF (CPU) fails. 
    }
    \label{fig:jackal_sim_online}
\end{figure}

\begin{table}[t]
\vspace{-0.1cm}
\caption{Quantiative measure ($\downarrow$ better) of number of \textbf{collisions} of jackal simulation experiments for online adaptation. 
Experiments that constantly lead to collisions are only executed $5$ times.
These measures highlight the benefits of \refineCBF and \algname.
}
\setlength{\tabcolsep}{3.5pt}
\begin{tabular}{c c c c c}
\refineCBF &
\refineCBF (CPU) &
\algname (CPU) &
jointCBFs &
HJR\\ \hline 
0/10 & 5/5 & 1/10 & 5/5 & 5/5
\end{tabular}

\label{tab:jackal_comp}
\vspace{-0.6cm}
\end{table}

\vspace{-0.3cm}
\subsection{Quadcopter: Backup CBF comparison} 
Backup Control Barrier Functions (Backup CBFs) are a method for ensuring safety by relying on a pre-defined, expert-chosen backup maneuver. 
A control action is deemed safe if the system can always execute this fixed backup policy (e.g., ``brake hard") for a set time without violating safety constraints~\cite{SingletarySwannEtAl2022}. 
This approach avoids complex online optimization but is inherently limited by the quality and applicability of the single, fixed backup policy. 
We compare against a variant that mixes nominal and backup controllers, as it is more computationally tractable. 

We demonstrate this limitation in a planar quadcopter simulation where the goal is occluded by an obstacle (Fig.~\ref{fig:drone_sim}), with dynamics:
\begin{align}\label{eq:planar_quad_no_disturbance}
    \dot{p}_y = v_y, \>\>\>\>\> \dot{p}_z = v_z, \>\>\>\>\> \dot{v}_y = g \tan(\phi), \>\>\>\>\> \dot{v}_z = T - g,
\end{align}
with inputs $u = [\phi, T]$. 
The drone uses a simulated front-facing camera, meaning the safety constraint requires it to remain within the currently visible, obstacle-free space. 
The nominal controller is a simple LQR designed to hover at the goal, which has no inherent obstacle avoidance capabilities. 
Our method, \refineCBF, is initialized with the signed-distance function $\constraintfunc(x)=\constraintfunc(p_y,p_z)$ of the initially known environment, which can be computed efficiently analytically (for polytopes) and through discretization using Fast Marching Methods~\cite{Sethian1999} (any environment).

The experiment highlights a critical failure mode for methods relying on fixed expert policies. 
The backupCBF method successfully maintains safety but its simple backup maneuver (here: stabilize to a hovering position) is insufficient to navigate the obstacle, resulting in deadlock (Fig.~\ref{fig:drone_sim}). 
In contrast, \refineCBF characterizes a safety filter over the full dynamics of the system. This obtains more informative gradients that are fully state-dependent, enabling the drone to navigate around the obstacle, as the safe set at a high velocity ``nudges'' the drone upwards, see Fig.~\ref{fig:sub2} (left), for a visualization. 

As new free space is revealed, \refineCBF seamlessly expands the safe set online (Fig.~\ref{fig:sub2}, center and right).
We do not compare against \algname in this scenario, as its formulation does not support the expansion of the safe set into previously unknown areas.
\begin{figure}
    \centering
    \includegraphics[width=\linewidth]{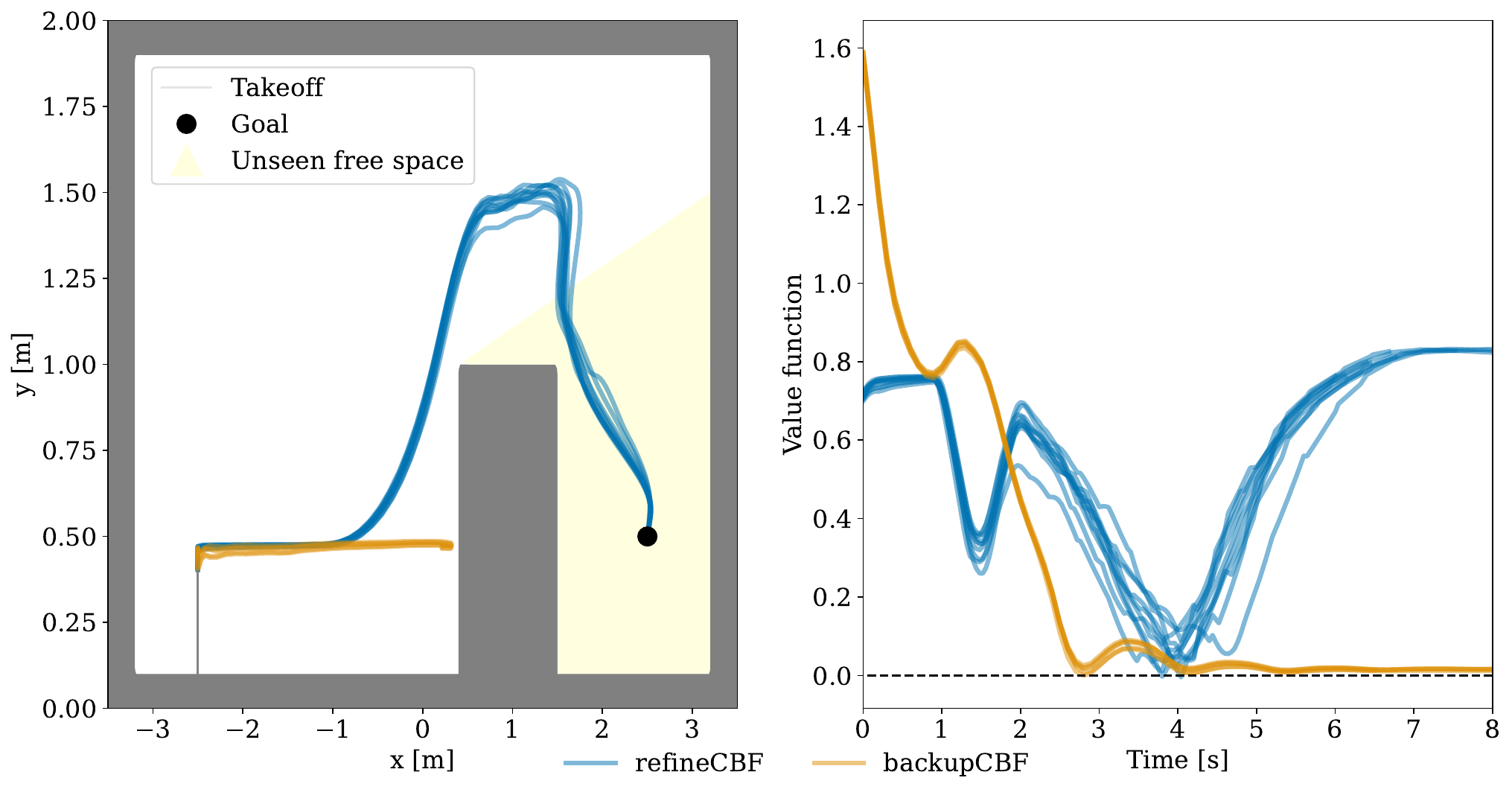}
    \caption{Online adaptation and deadlock avoidance in a quadcopter navigation task. The quadcopter must navigate to a goal in a region initially occluded by an obstacle. 
    (Left) Trajectories show \refineCBF successfully discovering a path over the obstacle, progressively incorporating the newly seen free space (yellow) into the safe set, then iteratively incorporating this new information into its safe set. 
    In contrast, the backupCBF method gets stuck in a deadlock, unable to find a path to the goal with its fixed backup policy. 
    (Right) The value function confirms this behavior. \refineCBF's value dips as it navigates the obstacle and unseen space boundary, and then recovers as the safe set expands. The backupCBF value drops and oscillates near the safety boundary (value=0), indicating a persistent deadlock. Further analysis on why \refineCBF avoids deadlock is provided in Fig.~\ref{fig:sub1}.}
    \label{fig:drone_sim}
    \vspace{-0.5cm}
\end{figure}

\section{Experiments (hardware)}\label{sec:hardware}
\begin{figure}[t]
    \centering
    \vspace{-0.8cm}
    \includegraphics[width=\linewidth]{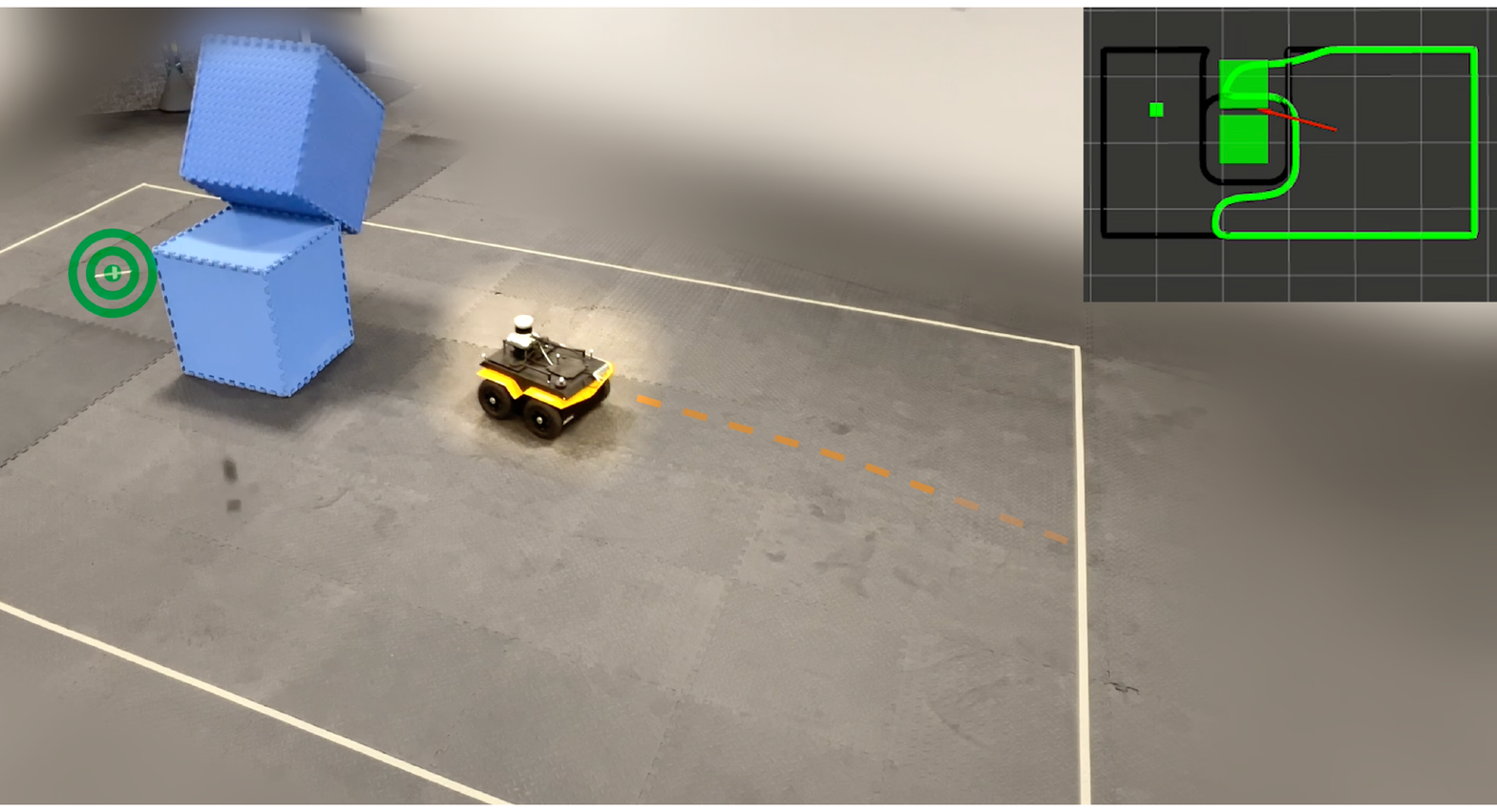}
    \vspace{-1.2cm}
    \caption{
    A snapshot of the Jackal hardware experiment. 
    The robot detects the suddenly fallen obstacle blocking its path. The top-right inset shows the robot's internal map at the moment of detection; the safe set (green) has not yet been updated to incorporate the new SDF map (black), highlighting the need for rapid, in-the-loop replanning to ensure safety.}
    \label{fig:jackal_hw_exp}
\end{figure}

\begin{figure}[t]
    \centering
    \includegraphics[width=\linewidth]{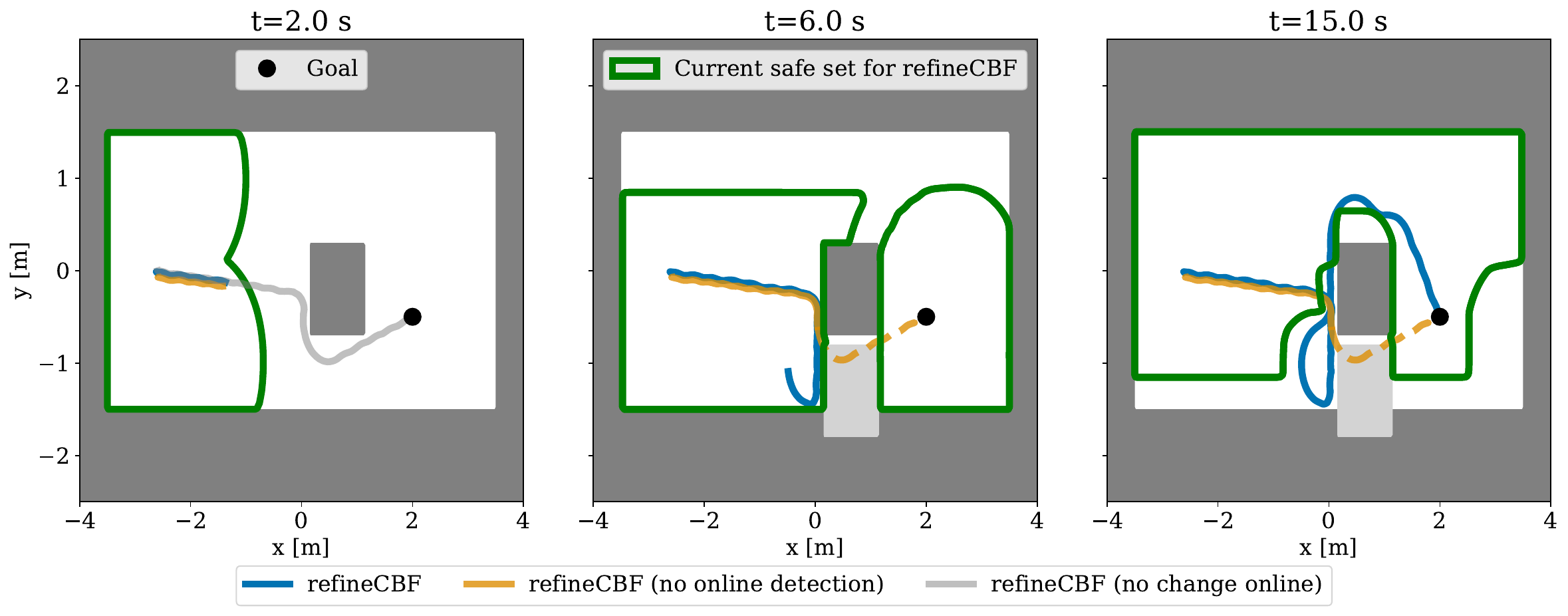}
    \caption{
    Trajectory comparison for the Jackal hardware experiment. 
    These plots show time snapshots of the system's trajectory after the environment is altered at $t=3.0s$. 
    The fully adaptive \refineCBF (blue) successfully backtracks to find a new, safe path. 
    The non-adaptive baseline (refineCBF (no online detection), orange) operates on a stale map, leading directly to a collision with the new obstacle.}
    \label{fig:jackal_hw}
\end{figure}
To validate our approach on a physical platform, these hardware experiments are designed to demonstrate the real-time adaptation capabilities of \refineCBF. 
\begin{itemize}
    \item Fully Adaptive \refineCBF (Our Method): Continuously refines its value function in response to environmental changes perceived in real time by the robot's sensors (simulated sensing). 
    \item Static \refineCBF (Baseline): Operates \refineCBF on a static representation of the environment captured only at initialization. This controller does not incorporate subsequent perceptual updates, highlighting the necessity of online adaptation for safety.
\end{itemize}
Both still iteratively update the value function online, but the baseline does not observe environmental changes in the loop. 
All experiments were conducted on a desktop computer with an NVIDIA RTX-4090 and 64GB of RAM. 
All controllers run at $50$Hz with the value function updates running in parallel with $\iter{\Delta}{k}=0.1$s for all iterations $k$ at a typical rate of $2-5$Hz.
\subsection{Mobile robots: Obstacle appears}
We validate our approach in a dynamic hardware experiment that mimics a challenging search-and-rescue scenario where the environment changes unexpectedly. 
We initialize the safe set using a CBF centered at the starting position $p_0$ with a small initially safe region, with a learned $\initialvf(x)$ using~\cite{DawsonQinEtAl2021} with $\constraintfunc(x)=\constraintfunc(p_x,p_y)$ a circle around $p_0$. 
A Clearpath Jackal UGV, modeled using~\eqref{eq:extended_unicycle}, is tasked with reaching a goal placed behind a two-tiered obstacle structure (Fig.~\ref{fig:jackal_hw_exp}). 
During operation, the top block is deliberately knocked over, blocking the robot's initial path and forcing an online adaptation.

As shown in Fig.~\ref{fig:jackal_hw}, the fully adaptive \refineCBF immediately perceives the new obstacle, backtracks, and the safety filter enables safely getting to the goal. 
In contrast, the non-adaptive baseline, which operates on a stale map of the initial environment, fails to see the change and collides with the fallen block. 
This result demonstrates that with a simple nominal controller, our online adaptation is critical for maintaining safety in dynamic environments. 
We present a single, representative run for each case, as this outcome was highly consistent across $5$ trials.

\subsection{Drone: Sensing limitations}
\begin{figure}[t]
    \centering
    \vspace{-0.8cm}
    \includegraphics[width=\linewidth]{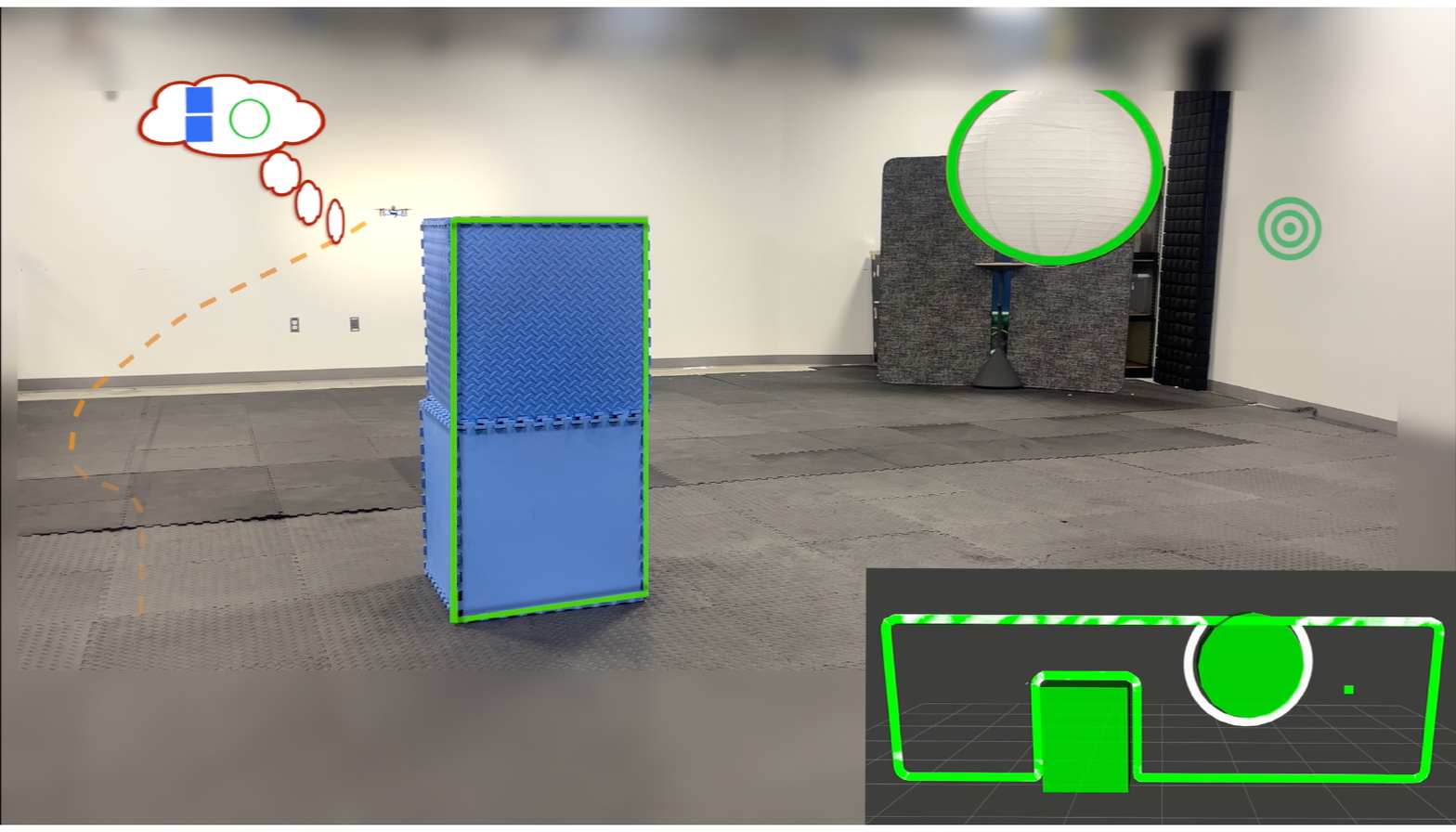}
    \vspace{-1.2cm}
    \caption{Quadcopter hardware experiment with a previously occluded obstacle. 
    Snapshot of the physical setup where a quadcopter navigates towards its goal (green circle). 
    The circular obstacle is initially hidden from view by the blue block. 
    The bottom-right inset shows the constraint and safe set at the moment the new obstacle is first detected; the safe set (green) has not yet been updated, highlighting the challenge for the real-time safety filter.
    }
    \label{fig:hw_exp_circle}
    \vspace{-0.3cm}
\end{figure}

\begin{figure}[t]
    \centering
    \includegraphics[width=\linewidth]{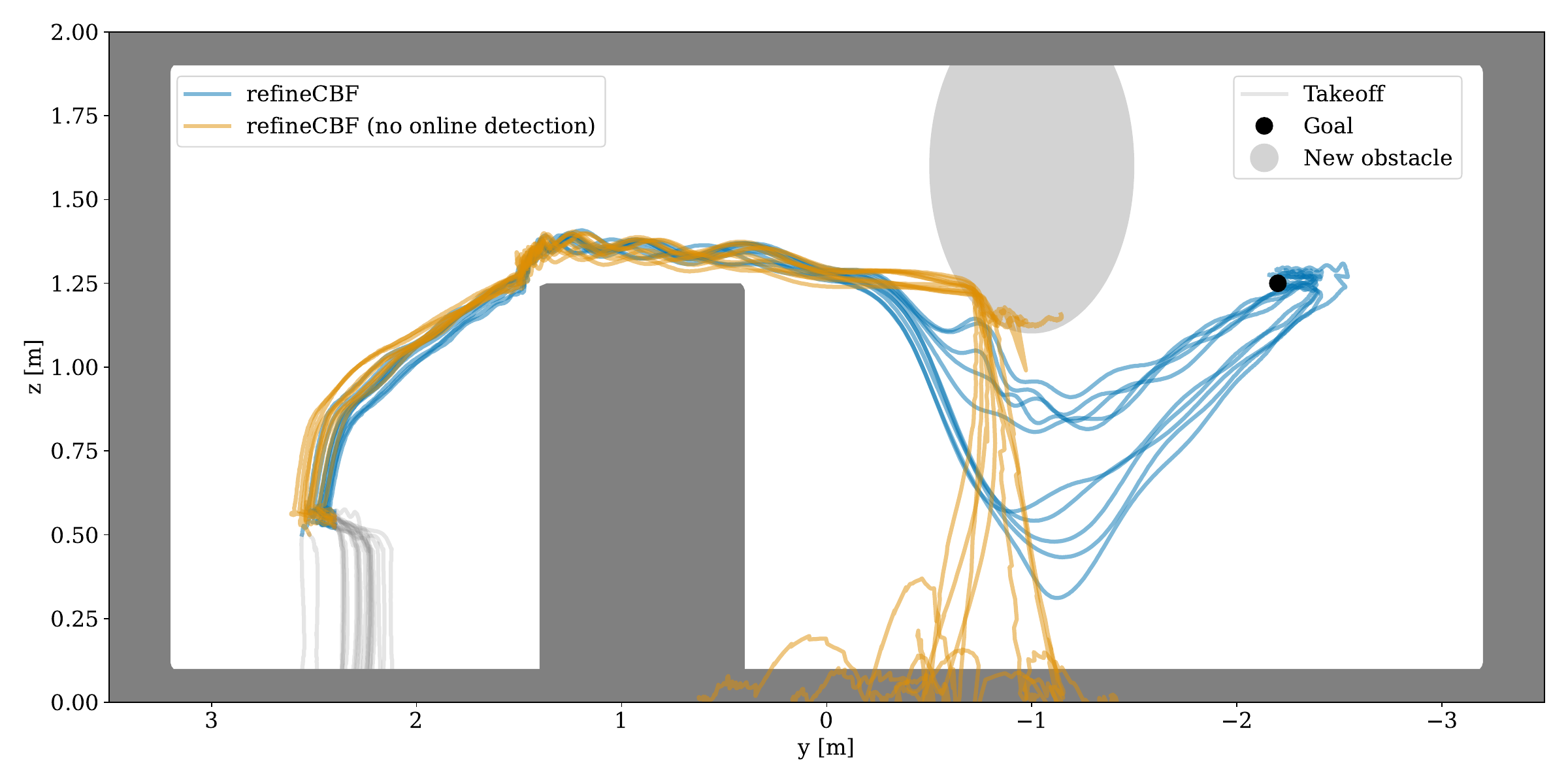}
    \caption{
    Trajectories from 10 hardware trials comparing adaptive and non-adaptive \refineCBF. 
    The quadcopter must navigate an environment with an unexpected obstacle (light gray circle) that only becomes visible after clearing the first obstacle. 
    The fully adaptive \refineCBF (blue) successfully avoids the new obstacle in all 10 trials. 
    The non-adaptive baseline (refineCBF (no online detection), orange), operating under the optimistic assumption that unseen space is free, collides in every trial.}
    \label{fig:cf_obstacle}
    \vspace{-0.6cm}
\end{figure}
Robots with forward-facing sensors like cameras or LiDAR have a limited field of view, forcing them to make assumptions about unobserved areas. 
A common and efficient strategy is to plan optimistically, assuming unobserved space is obstacle-free until proven otherwise. 
We designed a hardware experiment to demonstrate that \refineCBF's rapid online adaptation makes this optimistic approach both viable and safe. 
\begin{figure}[t]
    \centering
    \vspace{-0.8cm}
    \includegraphics[width=\linewidth]{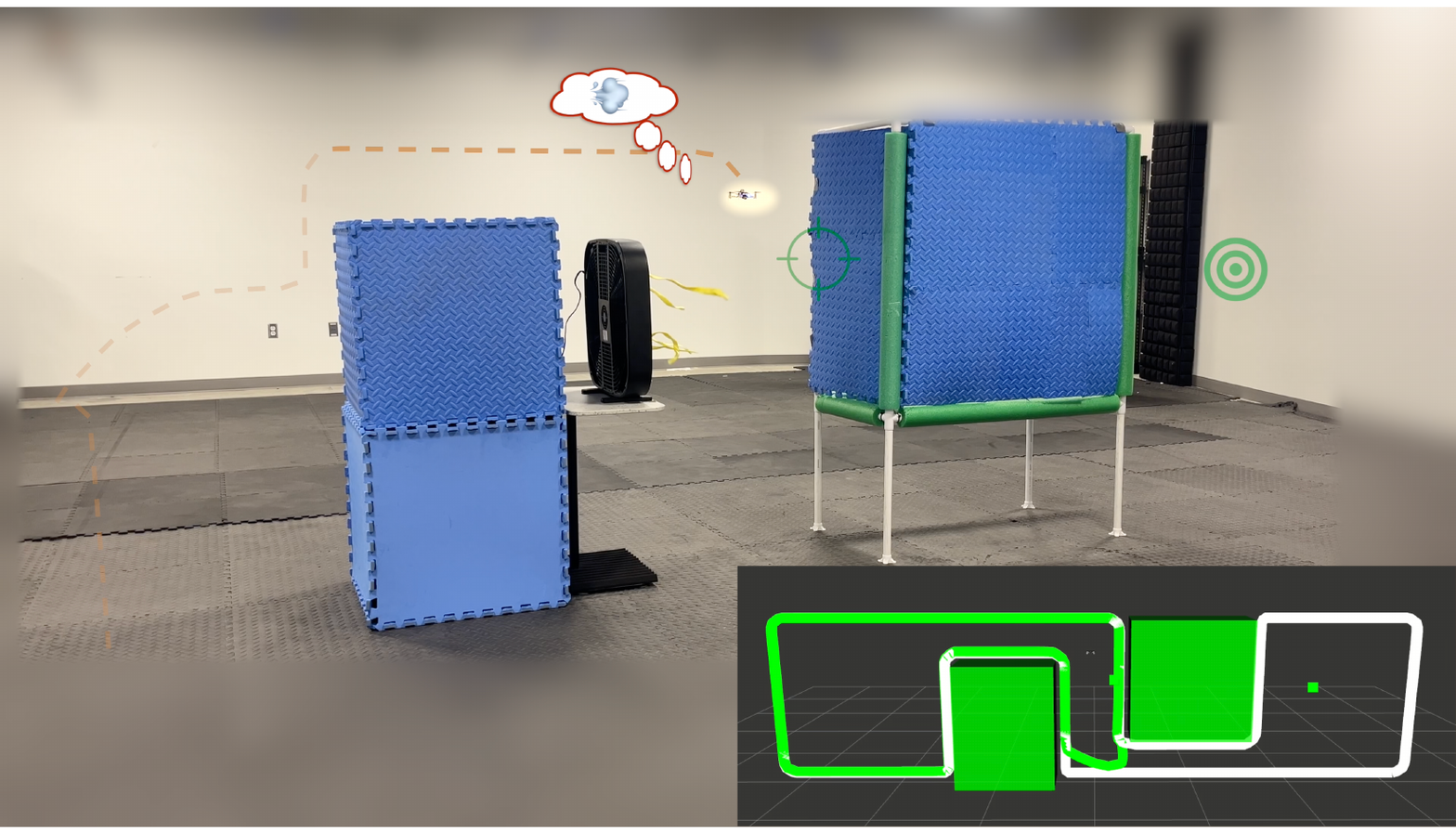}
    \vspace{-1.2cm}
    \caption{
    Hardware experiment with unmodeled aerodynamic disturbances. 
    A snapshot of the quadcopter navigating through a narrow passage containing a fan (wind visualized with yellow streamlines). 
    While the geometry of the passage is known, the aerodynamic disturbance produced by the fan's and its interaction with the environment is not and must be detected and compensated for in real time.}
    \label{fig:wind_hw_exp}
    \vspace{-0.3cm}
\end{figure}

\begin{figure}[t]
    \centering
    \includegraphics[width=\linewidth]{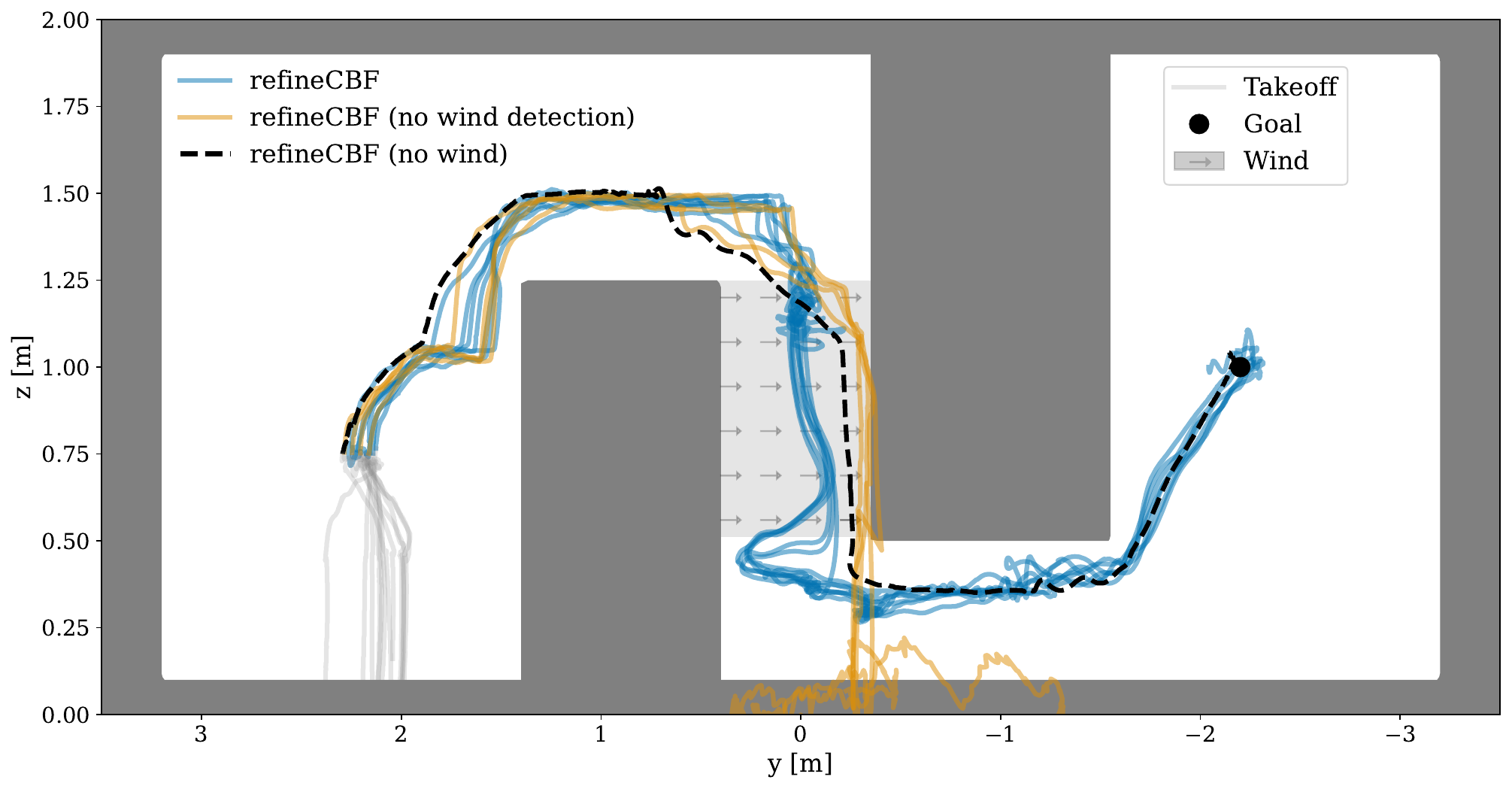}
    \caption{
    Trajectories from 10 hardware trials with a dynamic disturbance. The quadcopter encounters a strong wind disturbance while flying between the two known obstacles. 
    The safety value function is initialized as a CBF centered at the starting position with a small initially safe region. 
    The adaptive \refineCBF, which incorporates real-time disturbance changes, successfully compensates and navigates the passage in 9/10 trials. 
    The baseline controller (orange), which assumes nominal dynamics, fails in every trial, consistently being pushed into the wall by the unexpected disturbance. 
}
    \label{fig:cf_wind}
    \vspace{-0.7cm}
\end{figure}
In our scenario, a quadcopter, modeled using~\eqref{eq:planar_quad_no_disturbance}, must navigate to a goal that is initially occluded by a large rectangular obstacle (Fig.~\ref{fig:hw_exp_circle}). 
Hanging from the ceiling behind this obstacle is a second, circular obstacle, completely hidden from the robot's initial view. 
The initial value function $\initialvf(x)=\constraintfunc(x)$ is the SDF (computed as described in Section~\ref{sec:sim}).
Across 10 consecutive trials, the fully adaptive \refineCBF successfully perceived the new obstacle (recomputing the constraint function), updated its value function in real time (re-converging in under 2 seconds), and adjusted its trajectory to safely reach the goal. 
In contrast, the non-adaptive baseline, acting on its initial optimistic map, consistently collided with the unexpected obstacle, as shown in the trajectories in Fig.~\ref{fig:cf_obstacle}.

Defining the safety value function over all states enables navigating between obstacles without a sophisticated nominal policy, see Fig.~\ref{fig:sub2} for details.

\subsection{Quadcopter: Varying disturbances}

Our final hardware experiment tests \refineCBF's ability to adapt not just to geometric changes, but also to unmodeled dynamic disturbances.
In this scenario, a quadcopter navigates a fully known environment, but a fan creates a strong, localized wind disturbance within a narrow passageway (Fig.~\ref{fig:wind_hw_exp}). 
We consider the following dynamics:
\begin{align*}\label{eq:planar_quad_disturbance}
    \dot{p}_y = v_y + d_0, \>\>\> \dot{p}_z = v_z + d_1, \>\>\> \dot{v}_y = g \tan(\phi), \>\>\> \dot{v}_z = T - g,
\end{align*}
with inputs $u = [\phi, T]$ and disturbances $d=[d_0,d_1]$.  
We simulate a drone equipped with a flow sensor that allows it to detect and incorporate this disturbance in the safety analysis in real time, with a default value of $d_0,d_1\in\mathcal{D}_0=0$. 
The bounds $\mathcal{D}_\text{fan}$ are fitted a priori for different fan speeds. 
The initial value function $\initialvf$ is a learned CBF, using~\cite{DawsonQinEtAl2021}, centered around the initial hover point $p_0$.

Upon entering the passage and detecting the airflow, the fully adaptive \refineCBF rapidly updates its safety value function to account for these new dynamics. 
As the results from 10 trials show (Fig.~\ref{fig:cf_wind}), this allows the drone to successfully counteract the disturbance and traverse the passage in 9 out of 10 attempts. 
We hypothesize that the single failure occured due to the non-uniformity of the ``wind'' provided by the radial fan, which can destabilize the quadcopter. 
The non-adaptive baseline, operating on the assumption of nominal dynamics, is unable to compensate for the wind and consistently pushed into the obstacle.

For this experiment, we intialized the CBVF conservatively around its initial hover point, see Fig.~\ref{fig:sub3} (left). This experiment qualitatively demonstrates that the CBVF is able to expand from an initial safe but conservative safe set (Fig.~\ref{fig:sub3}, center and right). It leverages \saferefineCBF to ensure contraction upon detection of the wind disturbance. The nominal policy is a proportional controller with sub-goals dynamically updated to be the closest safe state (at $0$ velocity) to the final goal.

\section{Practical Use-Cases, Limitations \& Future Directions}\label{sec:future}
This work provides a step towards scalable, real-time safety adaptation for robotic systems. 
Our experiments demonstrate that by warm-starting from a prior value function, methods like \refineCBF and \algname can successfully adapt to sudden environmental changes, such as new obstacles or unmodeled dynamic disturbances (e.g., wind). 
This capability is critical for practical applications in evolving environments, including search-and-rescue, autonomous logistics, and long-term mobile robot deployment. 
However, several limitations highlight important avenues for future research. 

While promising, our approach has three principal limitations: idealized sensors, model dependency, and computational scalability.
\begin{enumerate}
    \item \textbf{Idealized Sensors}: In our hardware experiments, we focus on the controls and safety pipeline, and therefore simulated the sensing pipeline and used high-fidelity motion capture for state estimation. 
    This setup bypasses some significant challenges of a full-stack implementation, such as incorporating uncertainty inherent in state estimation and perception.
    \item \textbf{Known dynamics model:} Our methods fundamentally rely on an accurate known dynamics model with bounded uncertainties or disturbances, i.e. $\cldyn(x,u,d)$,
    limiting adoption to a subset of robotic domains. 
    \item \textbf{Scalability of dynamic programming:} The core of both \refineCBF and \algname is based on dynamic programming (DP), operating on a discretized state-space grid. 
    While effective for the systems we studied (up to 4D for on-the-fly refinement), relying on DP inherently limits scalability. 
\end{enumerate}
These limitations highlight several key avenues for future work. The most immediate practical step is to integrate realistic sensor modalities, developing methods that update the value function directly from noisy sensor data and explicitly account for perceptual uncertainty.

\begin{figure}[t]
    \centering

    \begin{subfigure}[b]{0.5\textwidth}
        \includegraphics[width=\linewidth]{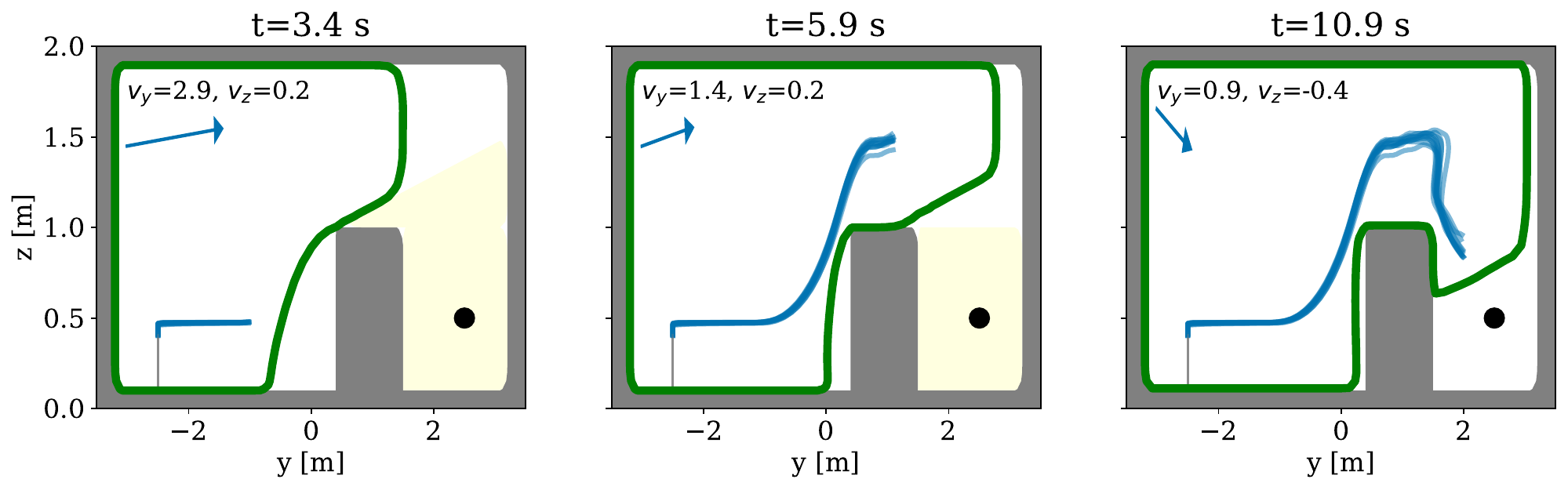}
        \caption{Visualized is the CBF expansion over time for the simulated Crazyflie UAV experiment. The CBVF is initialized as the SDF for the visible free space, and its corresponding safe set is shown in green. Note that the CBVF is in 4D; we therefore show the $y-z$ slice at the current $v_y,v_z$ values (indicated by the blue arrows) along the trajectory. The unseen region is marked in yellow. This experiment demonstrates the benefit of computing CBFs in higher dimensions: at higher speeds (left), the UAV must deviate from the square obstacle more quickly; this results in a CBVF whose gradients force the UAV upwards, avoiding the deadlock that occurs with position-based CBFs as in Fig~\ref{fig:drone_sim}.}
        \label{fig:sub1}
    \end{subfigure}
    
    \begin{subfigure}[b]{0.5\textwidth}
        \includegraphics[width=\linewidth]{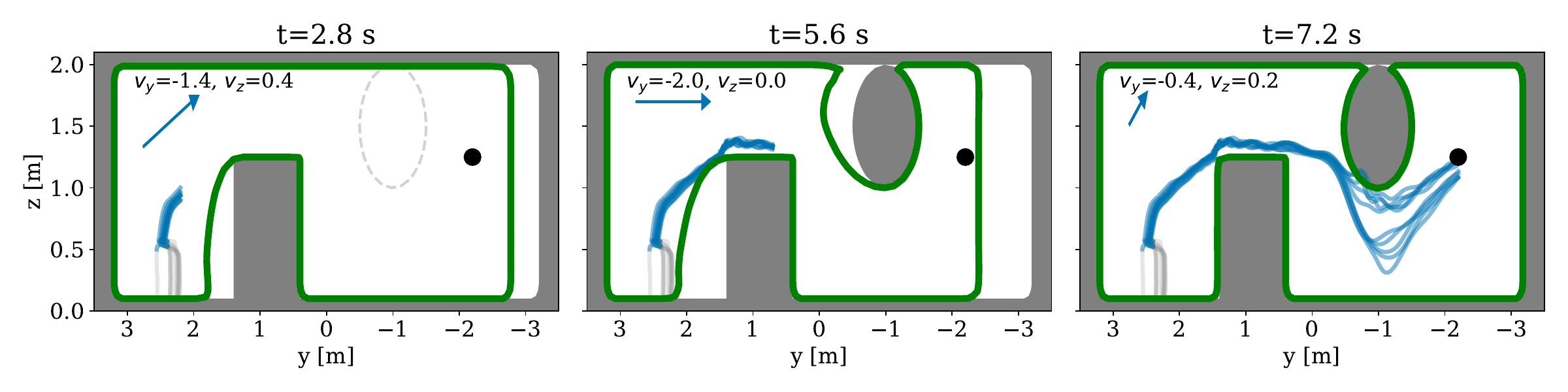}
        \caption{Visualized is the 2D slice of the safe set over time for the hardware Crazyflie UAV experiment. The CBVF is initialized as the SDF (without the unobserved circular obstacle). The safe sets in (left, center) are at high horizontal velocity, which results in the CBVF guiding the drone up and down, respectively, proactively providing safety by incorporating the full dynamics and state space information. 
        The comparison to a static \refineCBF is given in Fig.~\ref{fig:cf_obstacle}. 
        This once again highlights the importance of velocity information in the safety filter to avoid deadlock.
        }
        \label{fig:sub2}
    \end{subfigure}
    
    \begin{subfigure}[b]{0.5\textwidth}
    
        \includegraphics[width=1.0\linewidth]{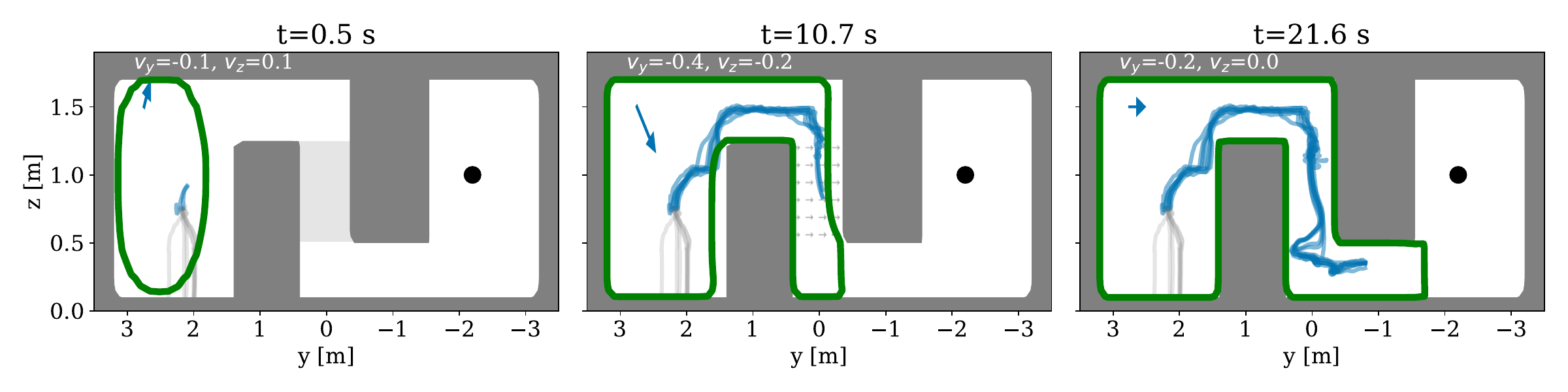}
        \caption{Visualized is the 2D slice of the safe set over time for the wind-affected hardware Crazyflie UAV experiment. The CBVF is initialized within a known conservative safe set (left).
        The safe set grows initially, then shrinks slightly upon observing the wind (center) leveraging \saferefineCBF. 
        The safe set continues to expand while providing control invariance after the wind corridor (right), before proceeding to its goal, see also Fig.~\ref{fig:cf_wind}.}
        \label{fig:sub3}
        
    \end{subfigure}

    \caption{Safe set across the position space evaluated at the current velocities (arrow and slice highlight the velocities). A broad range of value function initializations are presented (conservative SDF (top), optimistic SDF (center), and conservative CBF guess (bottom)), in addition to a broad range of online changes. 
    \refineCBF (or \saferefineCBF) succeeds in all scenarios, highlighting the diversity of scenarios where it can be deployed. 
    In addition, all experiments highlight that encoding safety through a value function over all states, not just positional states, is required in dynamic scenarios to preserve safety and maintain performance.}
    \label{fig:main}
\end{figure}

To address scalability, a promising direction is to merge our framework with modern learning-based approaches. While fully learned value functions have shown promise for high-dimensional systems, they are typically static and cannot adapt online. Some recent approaches have started to enable online variation by parameterizing the value function based on disturbances or obstacles~\cite{LiChen2025}. Building on this, a compelling research path is to synthesize our more general real-time adaptation techniques with these expressive, high-dimensional function approximators. Such a merger could combine the scalability of deep learning with the dynamic responsiveness demonstrated in our work, enabling safe, real-time adaptation for complex robots like manipulators and humanoids.

\section{Conclusions}\label{sec:conclusion}
This work demonstrates that HJ reachability, a cornerstone of formal safety analysis, can be made practical for real-time robotic applications when deployed with a CBF-like safety filter. 
We presented two algorithms, \refineCBF and its successor \algname, that effectively adapt safety-critical value functions to in-the-loop changes in a robot's environment or dynamics. 
By leveraging warm-starting, our methods rapidly re-converge to a safe policy following unexpected events, a capability we validated against multiple baselines in realistic simulations and on physical hardware. 
Our results highlight a key finding: warm-started HJ reachability provides a principled and effective framework for ensuring safety in dynamic, real-world settings, bridging the gap between theoretical guarantees and practical deployment.
\bibliographystyle{IEEEtran}

\bibliography{bibs/main, bibs/ASL_papers, bibs/SASLab}

\end{document}